\documentclass[twoside]{article}

\usepackage[accepted]{aistats2022}
\usepackage[round]{natbib}

\usepackage[utf8]{inputenc} % allow utf-8 input
\usepackage[T1]{fontenc}    % use 8-bit T1 fonts
\usepackage{hyperref}       % hyperlinks
\usepackage{url}            % simple URL typesetting
\usepackage{booktabs}       % professional-quality tables
\usepackage{amsfonts}       % blackboard math symbols
\usepackage{nicefrac}       % compact symbols for 1/2, etc.
\usepackage{microtype}      % microtypography
\usepackage{xcolor}         % colors

\usepackage{amsmath,amssymb,amsthm,amsfonts,bm,bbm}
\usepackage{mathtools}
\usepackage{cancel}

\usepackage{bbm}
\usepackage{nicefrac}       % compact symbols for 1/2, etc.

\usepackage{amsmath,amsfonts,bm}
\usepackage{amssymb,amsthm}
\usepackage{mathtools}
\usepackage{algorithm,algorithmic}
\usepackage{appendix}
\usepackage{enumerate}

\newtheorem{theorem}{Theorem}
\newtheorem{lemma}{Lemma}
\newtheorem{definition}{Definition}

\newenvironment{itemize*}%
{\begin{itemize}[leftmargin=*,topsep=0pt]%
		\setlength{\itemsep}{0pt}%
		\setlength{\parskip}{0pt}}%
	{\end{itemize}}

\newenvironment{enumerate*}%
{\begin{enumerate}[leftmargin=*,topsep=0pt]%
		\setlength{\itemsep}{0pt}%
		\setlength{\parskip}{0pt}}%
	{\end{enumerate}}

\def\shownotes{1} 
\ifnum\shownotes=1
\newcommand{\simon}[1]{[\textcolor{red}{Simon: #1}]}
\newcommand{\jnote}[1]{\textcolor{red}{[JL: #1]}}
\newcommand{\yuandong}[1]{[\textcolor{red}{YD: #1}]}
\newcommand{\yulai}[1]{[\textcolor{blue}{Yulai: #1}]}
\else
\newcommand{\simon}[1]{}
\newcommand{\jnote}[1]{}
\newcommand{\yuandong}[1]{}
\newcommand{\yulai}[1]{}
\fi

\usepackage{amsmath,amsfonts,bm}

\def\eqref#1{equation~\ref{#1}}
\def\1{\bm{1}}

\DeclareMathAlphabet{\mathsfit}{\encodingdefault}{\sfdefault}{m}{sl}
\SetMathAlphabet{\mathsfit}{bold}{\encodingdefault}{\sfdefault}{bx}{n}

\def\gA{{\mathcal{A}}}

\def\gC{{\mathcal{C}}}

\def\gF{{\mathcal{F}}}

\def\gH{{\mathcal{H}}}

\def\gM{{\mathcal{M}}}

\def\gO{{\mathcal{O}}}
\def\gP{{\mathcal{P}}}

\def\gS{{\mathcal{S}}}
\def\gT{{\mathcal{T}}}

\def\gW{{\mathcal{W}}}
\def\gX{{\mathcal{X}}}

\def\sR{{\mathbb{R}}}

\newcommand{\E}{\mathbb{E}}

\begin{document}
\twocolumn[

\aistatstitle{Provably Efficient Policy Optimization for Two-Player Zero-Sum Markov Games}

\aistatsauthor{ Yulai Zhao \And Yuandong Tian \And  Jason D. Lee \And Simon S. Du }

\aistatsaddress{ Tsinghua University \And  Meta AI Research \And Princeton University \And University of Washington\\
Meta AI Research} ]

\begin{abstract}
    Policy-based methods with function approximation are widely used for solving two-player zero-sum games with large state and/or action spaces.
However, it remains elusive how to obtain optimization and statistical guarantees for such algorithms.
We present a new policy optimization algorithm with function approximation and prove that under standard regularity conditions on the Markov game and the function approximation class, our algorithm finds a near-optimal policy within a polynomial number of samples and iterations.
To our knowledge, this is the first provably efficient policy optimization algorithm with function approximation that solves two-player zero-sum Markov games.
% \yulai{redundance? "provably efficient", "policy optimization algorithm with function approximation"...}

\end{abstract}

\section{Introduction}
\label{sec:intro}
Two-player zero-sum Markov game is a popular setting with many applications, such as Go~\citep{silver2016mastering}, StarCraft \uppercase\expandafter{\romannumeral2}~\citep{vinyals2019grandmaster}, and poker~\citep{brown2018superhuman}.
In this setting, the goal of player one is to find a policy that achieves the maximum reward against player two who plays optimally to minimize the reward in response to player one's policy.
%Our goal is to find a policy for each agent that reaches the \emph{Nash equilibrium} (NE), i.e., a pair of policies that no player has anything to gain by changing only its own strategy.

Policy optimization methods are widely used for solving zero-sum games. These algorithms often constrain the policy in a parametric form, and compute the gradient of the cumulative reward with respect to the parameters using the policy gradient theorem or its variants to update the parameters iteratively~\citep{sutton2000policy, kakade2002natural, silver2014deterministic}. 
Due to its flexibility, a wide range of successful results are attained by policy optimization methods. For example, \citet{lockhart2019computing} performed direct policy optimization against worst-case opponents and empirically demonstrate their effectiveness in Kuhn Poker and Goofspiel card game.
\citet{foerster2017learning} invented LOLA where each agent shapes the learning of other agents. It gave the highest average returns on the iterated prisoners' dilemma (IPD).

Despite the large body of empirical work using policy optimization methods for two-player zero-sum Markov games, theoretical studies are very limited.
In this paper, we aim to answer the following fundamental question:
\begin{center}
	\emph{Can we design a provably efficient policy optimization algorithm with function approximation	for two-player zero-sum Markov games with a large state-action space?}
\end{center}

We answer the above question affirmatively. We summarize our contributions below.
%     \item 

\paragraph{Our contributions.} 
 We design a new, provably efficient policy optimization algorithm for two-player zero-sum Markov games based on the natural policy gradient (NPG) method~\citep{kakade2002natural}. 
%  \yulai{
 On a high level, our algorithm has the two-step style as in previous work on value-based algorithms for two-player zero-sum Markov games~\citep{perolat2015approximate}.
In the \textbf{Greedy step}, we aim to find \emph{a pair of policies} that approximately solves matrix games for a given value function, and in the \textbf{Iteration step}, we aim to update the value function upon the current policy.
In contrast to their value-based algorithm where function approximation is used for value functions, our results are entirely policy-based. We only have \emph{function approximation for policies}.
Therefore, we need to tackle additional challenges which are absent in value-based algorithms.
\begin{enumerate}
\item First, in the \textbf{Greedy step}, the algorithms in~\citep{perolat2015approximate} requires solving two-player zero-sum matrix games \emph{for every state} in a value-based manner. The computational complexity scales with the size of the state-action space, which can be infeasible.
For finding the equilibria for state-wise matrix games, we employ policy-based methods whose sample complexity only scales with the complexity of the function class (e.g., feature dimension for linear function approximation) instead of the size of state-action space. Specifically, we design a subroutine that combines two-player policy gradients with the optimistic mirror descent (OMD) updates~\citep{rakhlin2013optimization} to solve the zero-sum matrix game in a computational and statistical efficient way.
% We obtain sub-optimality gaps for both settings, which could not be provided by previous value-based works. 
\item Second, in the \textbf{Iteration step},
\citet{perolat2015approximate} used \emph{Generalized Policy Iteration} to evaluate the value function while we only have function approximation for policies. 
We leverage recent developments on NPG in single-agent RL~\citep{agarwal2020optimality} to update \emph{policies} in this step and represent the value function using \emph{policies} instead of explicitly storing the value function. 
% Thus our results are flexible, and applicable even with infinite states.
\item Third, technically, we incorporates policy-based methods into value-based schemes, and develop new perturbation analyses for policy-based methods, both of which may be of independent interest.
\end{enumerate}
% }

Theoretically, first, to illustrate the main idea of our algorithm, in Section~\ref{sec:population_algo}, we study an idealized ``population" tabular Markov game setting where we can access the population quantities, including the true policy gradients and the Fisher information.
     We prove an $\widetilde{O}\left(\frac{1}{T}\right)$ rate,\footnote{$\widetilde{O}\left(\cdot\right)$ hides logarithmic factors.} where $T$ is the number of iterations.
  This result is interesting in its own right because this matches the rate in the single-agent RL setting~\citep{agarwal2020optimality}.
We further obtain an improved rate in the entropy-regularized setting~\citep{cen2020fast}.
% See please refer to Appendix~\ref{sec_ext_entropy} for details.
%     For population setting, we prove the convergence rate with respect to policy $\pi_1$ enjoys a fast rate $O\left(\frac{1}{T} + \frac{\log{T^\prime}}{T^\prime} \right)$ when assuming access to population quantities, like policy gradients. Here $T, T^\prime$ are numbers of iterations for two steps respectively. 
%     We use this setting as an example case to illustrate performance guarantees of NPG-based methods in solving Markov games.

We present our main algorithm and theoretical results in Section~\ref{sec:online_algo} where we study Markov games with a large state-action space, and log-linear policy parameterization is used for generalization.
Instead of the idealized ``population" setting, we study the realistic online setting, where we can only access the model through interactions.
%Second, we study the online setting, where we can only access samples from trajectories, and the MDP has a large state-action space where function approximation is used for generalization.
We prove an $\widetilde{O}\left(\frac{1}{\sqrt{T}}+\frac{1}{N^{1/4}} \right)$ rate where $T$ is the number of iterations and $N$ is the number of samples (interaction with the model). 
To our knowledge, this is the first quantitative analysis of online policy optimization methods with function approximation for two-player zero-sum Markov games.

\section{Related Work}
\label{sec:rel}

A large number of empirical works have proven the validity and efficiency of PG/NPG based methods in games and other applications~\citep{silver2016mastering, silver2017mastering, guo2016deep,mousavi2017traffic, tian2019elf}.
% e.g., AlphaGo used policy gradients with Monte Carlo Tree Search (MCTS), achieving 99.8\% winning rate over other Go programs~\citep{silver2016mastering}. 
% adopted deep policy gradients to predict best possible traffic signals in transportation.
Below we mostly focus on relevant algorithmic and theoretical papers.

There is a long line of work developing computationally efficient algorithms for multi-agent RL in Markov games.
% The existing algorithms providing non-asymptotic guarantees are mainly categorized into two classes: value-based and policy-based.
  Value-based approaches~\citep{shapley1953stochastic, patek1997stochastic, littman1994markov, bai2020provable, bai2020near} try to find the optimal value function. 
%  that is the unique fixed point under Bellman operator. 
% Generally joint policy at NE can be obtained via dynamic programming type methods. 
When the size of the state-action space is large, Approximate Dynamic Programming (ADP) techniques are often incorporated into value-based methods.
Extending the error propagation scheme of ADP developed by \citet{scherrer2012approximate} to two-player zero-sum games, \citet{perolat2015approximate} obtained a performance bound in general norms. Using this error propagation scheme on ADP, \citet{perolat2016use} adapted three value-based algorithms (PSDP, NSVI, NSPI) to the two-player zero-sum setting. 
 Recently, \citet{yu2019provable} replaced the policy evaluation step of Approximate Modified Policy Iteration (AMPI) introduced by \citet{scherrer2015approximate} with function approximation in a Reproducing Kernel Hilbert Space (RKHS) and proved linear convergence to $l_\infty$-norm up to a statistical error. While focusing on policy-based methods, our paper also leverages the error propagation analysis~\citep{perolat2015approximate}.

%\simon{TO discuss.}
Another type of algorithms on two-player zero-sum Markov games is policy-based.
% Instead of finding the fixed point of the Bellman operator, policy-based algorithms focus on maximizing the expected return of one agent disregarding other agents' policies.
One family of algorithms is based on fictitious play~\citep{brown1951iterative, robinson1951iterative}.
%\simon{@Yulai: our algorithm is Fictitious play?} \yulai{fictitious play finds best response of averaged policy (like regret?), we seek to find best response of the current policy. spirit similar? hmm}
Fictitious play is a classical strategy proposed by \citet{brown1951iterative}, where each player adopts a policy that best responds to the average policy of other agents inferred from historical data. 
For example, \citet{heinrich2015fictitious} introduced two variants of fictitious play: 1) an algorithm for extensive-form games which is realization-equivalent to its normal-form counterpart, 2) Fictitious Self-Play (FSP) which is
a framework computing the best response via fitted $Q$-iteration. 
%\simon{@Yulai: do they have theoretical guarantees?}\yulai{FSP should not have theoretical guarantees.}
Our paper also aims to find the best response iteratively.
Another family of policy-based methods is based on the idea of counterfactual regret minimization (CFR)~\citep{zinkevich2008regret}. 
%Compared with fictitious play algorithms whose convergence is analyzed asymptotically, explicit regret bounds are established with online learning tools.
 \citet{brown2019solving} invented a novel CFR variant which utilizes techniques such as  reweighting iterations and leveraging optimistic regret matching. 
Although these two families of algorithms are similar to ours in spirit, they are quite different technically and their theoretical analysis does not apply to our setting.
%These improvements led to better convergence rates to NE than CFR+, the prior state-of-the-art algorithm. However, theoretical results provided by CFR-type algorithms are regret bounds due to utilization of online learning tools, while our paper gives performance guarantee for the output policy at the last iteration.
%\simon{@Yulai: Do they deal with function approximation? Do they recover our guarantee?}\yulai{they do not have function-approx. Hmm, guarantee looks quite different with CFR-type. I think it's not comparable in essence.}

 The current paper focuses on using NPG techniques for solving two-player zero-sum Markov games. 
NPG is first introduced by \citet{kakade2002natural} to better explore the underlying structure of the reinforcement learning (RL) problem instance.
Extensions of NPG methods are also used to solve zero-sum games. \citet{zhang2019policy,bu2019global} applied projected natural nested gradient under a linear quadratic setting, a significant class of zero-sum Markov games.
Extensions to imitation learning were also studied in~\citep{song2018multi}.

In terms of theoretical analysis on PG/NPG methods, \citet{agarwal2020optimality} showed that tabular NPG could provide an $\gO(\nicefrac{1}{T})$ iteration complexity, as well as a sample complexity of $\gO(\nicefrac{1}{N^{\frac{1}{4}}})$ for online NPG with function approximation. 
In contrast, we provide bounds for the two-player zero-sum case, which is significantly more challenging.
% ~\yuandong{It would be great if we could provide some sentences on why two-player zero-sum games are harder} \yulai{
In two-player zero-sum games, the non-stationary environment faced by each individual agent invalidates the stationary structure of the single-agent setting, and thus precludes the direct application of the convergence proof from the single-agent setting.
Furthermore, each agent in two-player zero-sum games must adapt to the other agent's policy, which poses additional difficulties. 
% }
\citet{zhang2020global} proposed a new variant of PG methods that yielded unbiased estimates of policy gradients, which enabled non-convex optimization tools to be applied in establishing global convergence. Despite being non-convex, \citet{agarwal2020optimality,bhandari2019global} identified structural properties of finite Markov decision processes (MDPs): the objective function has no suboptimal local minimum. They further gave conditions under which any local minimum is near-optimal. 

\citet{schulman2015trust} developed a practical algorithm called TRPO which could be seen as a KL divergence-constrained variant of NPG. They show monotonic improvements of the expected return during optimization. \citet{shani2020adaptive} considered a sample-based TRPO~\citep{schulman2015trust} and proved an $\tilde{\gO}(\nicefrac{1}{\sqrt{N}})$ convergence rate to the global optimum, which could be improved to $\tilde{\gO}(\nicefrac{1}{N})$ when regularized. \citet{cen2020fast} showed that fast convergence rate of NPG methods can be obtained with entropy regularization. Applying NPG to linear quadratic games, \citet{zhang2019policy} and \citet{bu2019global} proved that: for finding Nash equilibrium, NPG enjoys sublinear convergence rate. Both analyses rely on the linearity of the dynamics which does not hold in general Markov games considered in this paper.

Recently, \citet{daskalakis2020independent} showed independent policy gradient methods converge to a min-max equilibrium. Compared to our work, they focused on the tabular case and did not study the function approximation. 
They also assumed that the probability of stopping at any state~\citep[Section 2]{daskalakis2020independent} is bounded below from a certain positive number, which is not a standard modelling approach and is hard to validate empirically. 
We instead use concentrability coefficients as a characterization of the game structure (cf. Definition~\ref{defn:concen}). In general, these two conditions do not imply each other. 
Comparing with their work, ours is cheap in sample complexity. To find an $\epsilon$-optimal solution, their sample complexity has an $O\left(\epsilon^{-12.5}\right)$ scaling whereas ours has an $O\left(\epsilon^{-6}\right)$ scaling.

Our work is related to Optimistic Mirror Descent 
(OMD) and its behavior in zero-sum games, which have received more attention lately. ~\citet{daskalakis2018training} proposed the use of optimistic mirror decent for training Wasserstein GANs to address the limit cycling problem in experiments. They also proved convergence to a equilibrium in bilinear zero-sum games. Generalizing~\citep{daskalakis2018training}, ~\citet{mertikopoulos2018optimistic} showed OMD converged in a class of non-monotone problems satisfying \emph{coherence}. Their work made concrete steps toward establishing convergence beyond convex-concave games.

\section{Preliminaries}
\label{sec:pre}
In this section, we introduce the material background on two-player zero-sum Markov games and specify several quantities which will be used to analyze our algorithms for different settings.

\subsection{Two-Player zero-sum Markov Games.}
% \yulai{
In this paper, we consider the centralized setting where we can control both players in the training phase to learn good policies.
% } 
we focus on infinite-horizon discounted two-player zero-sum Markov games, which can be described by a tuple $\gM = (\gS, \gA, \gP, r, \gamma)$: a set of states $\gS$, a set of actions $\gA$, a transition probability $\gP: \gS \times \gA \times \gA \to \Delta(\gS)$, a reward function $r : \gS \times \gA \times \gA \to [0, 1] $, and a discount factor $\gamma \in [0, 1)$. 
We let $\sigma$ to be the initial state distribution and
% \simon{@Yulai: what's our initial distribution? }\yulai{Our initial state distribution for training is defined as $\sigma$.}
% We follow notations of zero-sum matrix games from~\citep{rakhlin2013optimization} and 
define policies as probability distributions over the action space: $x, f \in \gS \to \Delta(\gA)$. \footnote{
% \yulai{
For clarity we assume two players share the same set of actions, it is straight forward to generalize to the setting where two action sets are different. See Section~\ref{sec:online_algo}.}
% } 
%\simon{@yulai: define what is a pair of policy mathematically.}
The value function $V^{x,f}: \gS \rightarrow \mathbb{R}$ is defined as:
\begin{gather*}
    V^{x,f} (s)\! = \!\mathop{\E} \limits_{ \substack{a_t \sim x(\cdot | s_t)\\ b_t \sim f(\cdot | s_t)\\ s_{t+1} \sim \gP( \cdot | s_t, a_t, b_t)} } \!\left[ \sum_{t=0}^{\infty}{\gamma^t  r(s_t, a_t, b_t) } \big{|} s_0 = s \right].
\end{gather*}
we use distribution $\sigma$ as the optimization measure we use to train the policy and use distribution $\rho$ as the performance measure of our interest. We remark that these two separate measures are widely used in analyzing approximate dynamic programming and policy gradient~\citep{agarwal2020optimality, perolat2015approximate}. We overload notations and define $V^{x, f}(\rho)$ as the expected value function of interest, i.e.
$ V^{x, f}(\rho) \coloneqq \E_{s \sim \rho}V^{x, f}(s).$

In a two-player zero-sum Markov game, player one ($x$) wants to maximize the value function and the other player ($f$) wants to minimize it. 
We define the Markov game's state-action value function $Q^{x, f}: \gS \times \gA \times \gA \to \mathbb{R}$, the advantage function $A^{x, f}:  \gS \times \gA \times \gA \to \mathbb{R}$, and the state visitation function $d_{s_0}^{x, f}: \gS \to [0,1]$ as
\begin{align*}
    Q^{x,f} (s,a,b) &= r(s, a, b) + \gamma \mathop{\E} \limits_{ \substack{s^\prime \sim \gP(\cdot | s,a,b)}} V^{x, f}(s^\prime),\\
    A^{x, f} (s,a,b) &= Q^{x, f}(s,a,b) - V^{x, f}(s),\\
    d_{s_0}^{x, f} (s) &= (1-\gamma)\sum_{t=0}^{\infty}\gamma^{t} Pr (s_t=s|s_0, x, f)
\end{align*}
where $s_0 \in \gS$ is an initial state, respectively. With the state visitation function at hand, we are prepared to introduce NPG~\citep{kakade2002natural} for two-player zero-sum games which relies on the Fisher information matrix.
Given player one's policy $x$, player two's policy $f$ parameterized by $\theta$, and starting state distribution $\sigma$, we define the Fisher information matrix $F_\sigma (\theta)$ as: 
\begin{gather*}
    F_\sigma (\theta) = \E_{s \sim d_\sigma^{x, f}} \E_{b \sim f(\cdot|s)} \nabla_{\theta} \log{f(b|s)} \nabla_{\theta} \log{f(b|s)}^\top,
\end{gather*}
where we denote $d_\sigma^{x, f} = \E_{s_0 \sim \sigma} d_{s_0}^{x, f}$ by the expectation form of the state visitation distribution.
% \simon{Did we define $d_\sigma$ before?}\yulai{yeah, see line 137-138. restate for clarity}

One important concept in RL is the Bellman operator. For two-player zero-sum Markov games and two behavior policies $x$ and $f$, we define $\gP_{x,f}(s^\prime|s) = \E_{a\sim x(\cdot|s), b\sim f(\cdot|s)} \gP(s^\prime|s,a,b)$ which performs as the transition kernel from $s$ to any $s^\prime\in \gS$ and $r_{x, f}(s) = \E_{a\sim x(\cdot|s), b\sim f(\cdot|s)}r(s,a,b)$ which represents the reward each player can expect
with policies $(x,f)$. Bellman operators $\gT_{x,f}, \gT_{x}, \gT$ act on any value function $v: \gS \to \sR$ and update it
\begin{itemize}
	\item $\gT_{x,f}v \coloneqq r_{x, f} + \gamma \gP_{x, f}v$, which generalizes the standard Bellman operator. 
%	Note that it is a linear operator with fixed point $v_{x, f}$
	\item $\gT_{x}v \coloneqq \inf_{f}  \gT_{x, f}v$, which is an asymmetric operator  by letting $f$ to be optimal. \footnote{Here we focus on max player $x$ (see Eq.~\ref{eq_metric}). If we replace $x$ by $f$, we will have an analogous notation for min player.}
	\item $\gT v \coloneqq \sup_{x} \gT_{x} v = \sup_{x} \inf_{f} \gT_{x, f } v $, which generalizes the standard Bellman optimality operator. It reflects notions of minimax equilibrium in essence.
\end{itemize}

% For simplicity, we also use $V^{x}$ for $\inf_{f} V^{x, f}$.
\citet{perolat2015approximate} introduced these operators as generalized counterparts of single agent RL. 
%One can see that player one intends to maximize cumulative discounted reward while player two intends to minimizeit. 
We are able to adopt the dynamic programming scheme only once the Bellman operators are introduced.
%\simon{@Yulai: please add sa one sentence explaination for each operator above.}
%\yulai{Added, try to explain in regular RL language..}

Since we are considering a learning problem, we need to collect samples from the environment. 
We assume we can stop and restart at any time.
With this, we can have the following sampling oracle.
% Our restarting sampling oracle is specified as
\paragraph{Episodic Sampling Oracle} \label{sec5_assump} For a fixed state-action distribution $\nu_0$, we can start from $s_0, a_0, b_0 \sim \nu_0$, act according to any policy pair $(x, f)$, and terminate when desired. We obtain unbiased estimates of the \emph{on policy} state-action distribution
\begin{align} \label{eq_stateaction_distribution}
      &\nu_{\nu_0}^{x, f} (s, a, b ) = (1-\gamma) \cdot \notag\\
      &\E_{s_0, a_0, b_0 \sim \nu_0} \sum_{t=0}^\infty \gamma^t \text{Pr} (s_t=s, a_t=a, b_t =b | s_0, a_0, b_0)
\end{align}
which can be used for acquiring an unbiased $Q^{x,f}(s,a,b)$ where $s, a, b \sim \nu_{\nu_0}^{x, f}$. See \citep[Algorithm 1]{agarwal2020optimality} for a sampler. 

This oracle essentially requires that we can terminate at any time and restart, therefore many real-world applications including games and physics simulation (e.g., OpenFOAM~\citep{weller1998tensorial}) admit this oracle.
This oracle is also used in the analysis~\citep{agarwal2020optimality} (see Algorithm 1,3 and Assumption 6.3 therein).

% \yulai{
The oracle is essential in analysis, technically because policy gradient methods need to estimate the values. This is the same reason as in the single-agent setting~\citep{agarwal2020optimality}. Moreover, we believe this sampling oracle is not a strong assumption: it only requires that we can terminate at any time and restart. This is much weaker than the generative model assumption. The generative model assumes that one can query \emph{any} state-action pair where we only require we can restart from a fixed initial distribution.
% }

\citet{shapley1953stochastic} show that $(x^*,f^*)$ is a pair of Nash equilibrium (NE) if the following inequalities hold for any state distribution $\rho$ and policy pair $(x,f)$:
\begin{align}\label{sec3_NE_equa}
    V^{x, f^*}(\rho) \le V^{x^*, f^*}(\rho) = V^*(\rho) \le V^{x^*, f}(\rho).
\end{align}
NE always exists for discounted two-player zero-sum Markov Games~\citep{filar2012competitive}.
In practice, we seek to find an approximate pair of NE instead of an exact solution. The goal of this paper is to output a policy $x$ that makes the metric
\begin{equation} \label{eq_metric}
    V^*(\rho) - \inf_{f} V^{x, f}(\rho)
    \vspace{-1em}
\end{equation}
small where $\rho$ is some state distribution of interest.
This metric measures the performance of $x$ against the worst-case $f$. If it is less than $\epsilon$, we call $x$ an one-sided $\epsilon$-approximate NE, \footnote{From an optimization perspective, the sampling complexity of finding a solution so that both the min and max player are approximate NE scales only twice as large as that in one-sided case, since we may apply algorithms with the roles switched.} it has been used by~\citep{daskalakis2007progress, goos2018near, deligkas2017computing, babichenko2020communication,daskalakis2020independent}. 

\subsection{Function Approximation} \label{sec_FA}
This paper studies function approximation to generalize across a large state space in Section~\ref{sec:online_algo}.
To represent both behavior policies $x$ and $f$, we adopt a log-linear parameterization: for a coefficient vector $\theta \in \sR^{d}$, the associated probability of choosing action $a$ under state $s$, $\pi_\theta (a|s)$, is given by
$\frac{\exp{(\theta^\top \phi_{s,a})}}{\sum_{a^\prime \in \gA} \exp{(\theta^\top \phi_{s, a^\prime})}}$
where $\phi_{s,a}$ is a feature vector representation of $s$ and $a$. This parameterization has been used in~\citep{branavan2009reinforcement, gimpel2010softmax, heess2013actor}.
We impose a regularity condition such that every $\|\phi_{s,a}\|_2 \le D.$ Note that log-linear parameterization is $D^2$-smooth in terms of $\theta$~\citep{agarwal2020optimality}. \footnote{We follow standard smoothness definition. A function $f$ is said to be $\beta$-smooth if for all $x, x^\prime \in \sR^d$: $\|\nabla f(x) - \nabla f(x^\prime)\|_2 \le \beta \|x - x^\prime\|_2.$} This term is also known as \emph{Policy Smoothness} when analyzing PG methods.

\subsection{Problem-Dependent Quantities.}
Our analysis relies on several problem-dependent quantities. 
We denote weighted $L_p$-norm of function $f$ on state space $\gS$ as $\|f\|_{p, \rho} = \left( \sum_{s \in \gS} \rho(s) |f(s)|^p \right)^{\frac{1}{p}}$.

The first problem-dependent quantity is used to measure the inherent dynamics of Markov games.
\begin{definition} [Concentrability Coefficients]
	\label{defn:concen}
	 Given two distributions over states: $\rho$ and $\sigma$. When $\sigma$ is element-wise positive, define
\begin{align*}
        c_{\rho, \sigma}(j) &= \sup_{x^1, f^1, \cdots x^j, f^j \in \gS \xrightarrow[]{} \Delta(\gA)} \left\| \frac{\rho \gP_{x^1, f^1}  \cdots \gP_{x^j, f^j} }{ \sigma} \right\|_{\infty},\\
        \gC_{\rho, \sigma}^\prime &= (1-\gamma)^2 \sum_{m \ge 1} m \gamma^{m-1} c_{\rho, \sigma}(m-1),\\
        \gC_{\rho, \sigma}^{l, k, d} &= \frac{(1-\gamma)^2}{\gamma^l-\gamma^k} \sum_{i=l}^{k-1} \sum_{j=i}^\infty \gamma^j c_{\rho, \sigma}(j+d).
\end{align*}
\end{definition}
Here, $x^1, f^1, \cdots x^j, f^j $ are $j$ pairs of policies. Intuitively, the first term quantifies the distribution shift after taking $j$ pairs of steps starting from $\rho$. The second term describes the accumulative effect of discounted distribution shifts. Finally, the last term represents the additive performance of $(k-l)$ accumulative distribution shift and thus it is often considered as stricter condition. See~\citep{scherrer2014approximate} for a thorough comparison on these coefficients.
% }
Generally speaking, if $\sigma$ is sufficiently diverse across states, then these quantities are bounded from above.
\citep{chen2019information} pointed out that small concentrability coefficients reflect a restriction on the MDPs dynamics.
Concentrability coefficients are widely used in analyzing the convergence of approximate dynamic programming algorithms~\citep{munos2005error, antos2008learning, scherrer2014approximate, perolat2015approximate} and recently in analyzing PG methods~\citep{agarwal2020optimality}.
In particular,~\citep{agarwal2020optimality} gave an example to show the \emph{dependency on concentrability coefficients is necessary}.
In these papers, their upper bounds all depend on the concentrability coefficients.
For our two-player setting, we use the same definition of concentrability coefficients as~\citep{perolat2015approximate}.

The second quantity measures how well a parameterized class can approximate in terms of a metric.
\begin{definition}[Approximation Error] 
\label{defn:approx_error}
Given a space $\gW$ and a loss function $L: \gW \rightarrow R$, we define $\epsilon_{approx} = \min_{w \in \gW} L(w)$ as the approximation error of $\gW$.
\end{definition}

This concept is widely used for analyzing function approximation~\citep{menache2005basis, jiang2015abstraction}, state abstractions schemes~\citep{jiang2015abstraction} and representation learning in RL~\citep{bellemare2019geometric}. It explicitly describes the capacity of a parameter set.

\section{Warm-up: Population Algorithm for Tabular Case} 
\label{sec:population_algo}
% \simon{mention fictitious play}
We first introduce the population version algorithm for the tabular case with the exact Fisher information matrix and policy gradients.
The algorithm is spiritually similar to fictitious play.
We enforce $x, f$ to be tabular softmax parameterized by $\xi, \theta \in \sR^{|\gS| \times |\gA|}$
\paragraph{Parameterization}
For vector $\theta \in \sR^{|\gS| \times |\gA|}$, the probability associates to choosing action $a$ under state $s$, $\pi_\theta (a|s)$, equals
$
    \frac{\exp{(\theta_{s,a})}}{\sum_{a^\prime \in \gA} \exp{(\theta_{s, a^\prime})}}.
$ One can verify that $\pi_\theta$ is $1$-smooth in terms of $\theta$.

This algorithm can be viewed as a prototypical algorithm and in the subsequent section, we will generalize to the online setting.
The pseudo-code is listed in Algorithm~\ref{alg:popu_NPG}.
\setlength{\textfloatsep}{0.2cm}
\setlength{\floatsep}{0.2cm}
\begin{algorithm}[t]
\begin{algorithmic}
    \caption{Population Two-Player NPG.} 
    \label{alg:popu_NPG}
    \REQUIRE $V_0= 0$ a value function.
    \ENSURE Approximate policy $x^K$ at Nash equilibrium
    \vspace{-1em}
    \FOR{$k=1,2,\cdots,K$}
        \STATE \textbf{Greedy Step:}
         \STATE Run Algorithm~\ref{alg:popu_MD} with $A_s$ defined in Eq.~\ref{eq_greedy_step} and returns $x^k(\cdot|s)$ for every state $s$.
         \STATE \textbf{Iteration Step:}
        \STATE Fix $x = x^k$, initialize $\theta=0$.
            \FOR{$t=0,1, \cdots,T-1 $}
                \STATE $\theta^{t+1} = \theta^t - \eta F_\sigma (\theta^{t})^\dag \nabla_{\theta} V^{x, f^t}(\sigma)$ \label{eq_popu_fisher_update}.
            \ENDFOR
        \STATE $V_{k} = V^{x, f^T}$.
    \ENDFOR
\end{algorithmic}
\end{algorithm}

\setlength{\textfloatsep}{0.2cm}
\setlength{\floatsep}{0.2cm}
\begin{algorithm}[t]
\begin{algorithmic}
    \caption{Subroutine: OMD for tabular case}
    \label{alg:popu_MD}
    \REQUIRE $f_0, g_0^\prime, x_0, y_0^\prime \in \text{Unif}(\gA)$, $\beta = \frac{1}{{T^\prime}^2}$, and $A_s$ for $s \in \gS$.
    \ENSURE Approximate optimal $\Bar{x_{T^\prime}}$ for max player
    \FOR{$t=1,2,\cdots,T^\prime$}
        \STATE \textbf{min player:}
        \STATE play $f_t(\cdot|s)$, observe $A_s^\top x_t(\cdot|s)$. Update:
         \vspace{-1em}
         \begin{gather*}
             g_t(i) \propto g_{t-1}^\prime (i) e^{ -\eta_t [x_t^\top A]_i },~
            %  \alpha_t(\cdot|s) = \alpha{t-1}^\prime (\cdot|s) - \eta_t^s A_s^\top x_t(\cdot|s) \\
         g_t^\prime = (1-\beta)g_t + \frac{\beta}{|\gA|} \textbf{I},\\
        %  \alpha_t^\prime(\cdot|s) = \ln{\left( (1-\beta)e^{\alpha_t(\cdot|s)} + \left( \frac{\beta}{|\gA|} \sum_j e^{\alpha_t^j(\cdot|s)} \right) \textbf{I} \right) }\\ 
          f_{t+1}(i) \propto g_t^\prime (i) e^{ -\eta_{t+1} [x_t^\top A]_i }
        %   \theta_{t+1}(\cdot|s) = \alpha_t^\prime(\cdot|s) - \eta_{t+1}^s A_s^\top x_t(\cdot|s)
         \end{gather*}
         \vspace{-1.5em}
         \STATE \textbf{max player:}
        \STATE play $x_t(\cdot|s)$, observe $A_s f_t(\cdot|s)$. Update: \vspace{-1em}
         \begin{gather*}
             y_t(i) \propto y_{t-1}^\prime (i) e^{ -\eta_t^\prime [A f_t]_i },~
            %  \zeta_t(\cdot|s) = \zeta_{t-1}^\prime(\cdot|s) - {\eta_t^s}^\prime A_s f_t(\cdot|s)\\
             y_t^\prime = (1-\beta)y_t + \frac{\beta}{|\gA|} \textbf{I},\\
            %  \zeta_t^\prime(\cdot|s) = \ln{\left( (1-\beta)e^{\zeta_t(\cdot)} + \left( \frac{\beta}{|\gA|} \sum_j e^{\zeta_t(\cdot|s)^j} \right) \textbf{I} \right) }\\ 
         x_{t+1}(i) \propto y_t^\prime (i) e^{ -\eta_{t+1}^\prime [A f_t]_i } 
        %  \xi_{t+1}(\cdot|s) = \zeta_t^\prime(\cdot|s) - {\eta_{t+1}^s}^\prime A_s f_t(\cdot|s)
         \end{gather*}\vspace{-2em}
    \ENDFOR
\end{algorithmic}
\end{algorithm}

In Algorithm~\ref{alg:popu_NPG}, we perform $K$ outer loops and obtain a near-optimal $x$ and value function $V_K$. 
We note that this algorithm is asymmetric since our metric (Eq.~\ref{eq_metric}) is only considering max player $x$ while taking the best response of min player $f$. 

% \simon{To discuss}
Each outer iteration begins with a \textbf{Greedy Step}. For current $V_{k-1}$, we aim to find approximate equilibrium ($x,f$) with which $\gT_{x, f} V \approx \gT V_{k-1}$.  This step is spiritually equivalent to finding minimax equilibrium of a matrix game for \emph{every state $s$}. In intuition, this step helps to update $V_{k-1}$ towards $V^*$ (cf. contraction Lemma).

Let us take a closer look at \textbf{Greedy Step}. Consider an approximate two-player zero-sum matrix game: for every state $s\in \gS$, we try to solve
\begin{gather} 
\label{eq_greedy_step}
    \max_{x(\cdot \mid s) \in \Delta(\gA)} \min_{f(\cdot \mid s) \in \Delta(\gA)} x^\top A_s f,\\
    A_s(a, b) = r(s, a, b) + \sum_{s^\prime} \gP(s^\prime \mid s, a, b) V_{k-1}(s^\prime). \notag
\end{gather}

\vspace{-1em}
Here $A_s$ represents a set of matrices related to current value function $V_{k-1}$. Instead of value-based approaches (PI~\citep{patek1997stochastic}, VI~\citep{shapley1953stochastic}) which are often  inefficient, we solve these matrix games by policy-based methods for efficiency and sub-optimality guarantee. We adopt the \emph{Optimistic Mirror Descent}~\citep{rakhlin2013optimization}  for two players by assuming access to population quantities, e.g., $A_s \pi$ in Eq.~\ref{eq_greedy_step}. Note that each $A_s(a,b) \in [0, \frac{1}{1-\gamma}]$ because $V_{k-1} \in [0, \frac{1}{1-\gamma})$.

For clarity, we follow notations in~\citep{rakhlin2013optimization}. Denote $\phi(f, x) = x^\top A_s f $ which is convex w.r.t. $f$ when fixing $x$ and concave w.r.t. $x$ when fixing $f$, and the domains for $x, f$ are $\gX, \gF$ respectively. Thus $\gT V_{k-1}(s) \coloneqq \sup_{x\in \gX} \inf_{f\in \gF} \phi(f,x)$. we denote $\{y_t\}$ and $\{g_t\}$ as secondary sequences of $\{x_t\}$ and $\{f_t\}$ respectively. We refer readers to Appendix~\ref{sec:pf_popu} for how we set the adaptive stepsizes $\eta_t$ and $\eta_t^\prime$. 

We perform simultaneous updates for $T^\prime$ iterations in Algorithm~\ref{alg:popu_MD} to minimize the following terms,
\begin{gather} \label{eq_MD}
    \frac{1}{T^\prime} \sum_{t=1}^{T^\prime} \phi(f_t, x_t) - \inf_f \sum_{t=1}^{T^\prime} \frac{1}{T^\prime} \phi(f, x_t),\\
% \text{ and } 
          \frac{1}{T^\prime} \sum_{t=1}^{T^\prime} (-\phi(f_t, x_t)) - \inf_x \frac{1}{T^\prime} \sum_{t=1}^{T^\prime} (- \phi(f_t, x)).
  \end{gather}
%suppose two infs are achieved at $f^*$ and $x^*$ respectively. With these two inequalities, we can derive an upper bound of \textbf{Greedy Step}  from Eq.~\ref{Eq_online_sandwich}:
Suppose two infs are achieved at $f^*$ and $x^*$ respectively. With these two inequalities, we can derive an upper bound of \textbf{Greedy Step}:
% \simon{Use either Equation or Eq. consistently throughout the paper.}
\begin{align*}
    \sup_{x \in \gX} \inf_{f \in \gF} \phi(f, x) - \inf_{f\in \gF} \phi(f, \Bar{x_{T^\prime}} )
\end{align*}
to guarantee $x^k$ is near-optimal with respect to $V_{k-1}$.

%\simon{Need more detail. Need to refer to the algorithm.}

After obtaining $x^k$ from \textbf{Greedy Step}, the \textbf{Iteration Step} aims to evaluate the value function while fixing $x= x^k$. We run $T$ updates to find $f^* = \arg \min_{f} V^{x, f}$. In the competitive multi-agent RL literature, this step is equivalent to finding the \emph{best response} of min player (namely, $f^*$) when fixing $x = x^k$. The intuition is that when the max player's policy is very close to its optimal policy at NE and $f$ takes $f^*$, their accumulative value function is also close to $V^*$ at NE.
This step can be viewed as running NPG for a single-agent RL problem.

The following theorem gives the performance guarantee for Algorithm~\ref{alg:popu_NPG}. 
% \yuandong{What's the intuition of $\rho$ and $\sigma$? How they are defined?}
% \yulai{discussed a little after Theorem.}

\begin{theorem} \label{thm:popu_NPG}
    For Algorithm~\ref{alg:popu_NPG}, set $\eta \ge (1-\gamma)^2 \log|\gA|$. After $K$ outer loops we have $V^*(\rho) - \inf_{f} V^{x^K, f}(\rho)$ upper bounded by
    \begin{align*}
       \widetilde{O}\left(\frac{ \gC_{\rho, \sigma}^{1, K, 0}}{(1-\gamma)^4 T}
        \!+ \!\frac{ \gC_{\rho, \sigma}^{0, K, 0} }{(1-\gamma)^4 T^\prime} \log{T^\prime}
        \!+ \!\frac{ \gamma^K}{1-\gamma} \gC_{\rho, \sigma}^{K, K+1, 0}\right).
    \end{align*}
\end{theorem}
\vspace{-1em}
We remind that $\sigma$ is the optimization measure we use to train the policy and $\rho$ is the performance measure of our interest.
%Also see~\citep{agarwal2020optimality} for discussions.
%Distribution $\sigma$ holds wide coverage over all states to guarantee exploration capacity, while distribution $\rho$ is for testing. We wish to observe good performance under $\rho$ with algorithm trained on $\sigma$. 
% \simon{Explain what are those concentrablity coefficients terms. Do these terms appear in~\citep{perolat2015approximate}?}\yulai{very hard to explain... They do appear in~\citep{perolat2015approximate}.}

Theorem~\ref{thm:popu_NPG} explicitly characterizes the performance of the output $x^K$ in terms of the number of iterations and the concentrability coefficients.
Viewing concentrability coefficients to be constants (which is the case when $\sigma$ is sufficiently diverse) and looking at the dependency on $T$ and $K$, we find the dependency on $T$ is a fast $1/T$ rate, matching the same rate in the single agent NPG analysis~\citep{agarwal2020optimality}.
The dependency on $K$ is exponential $(\gamma^K)$ which means we only need a few outer loops.
The first term has an $(1-\gamma)^{-4}$ dependency on the discount factor, which may not be tight and we leave it as a future work to improve.
% \yulai{
In Theorem~\ref{thm:popu_NPG} and ~\ref{thm:popu_NPG_entropy}, we set $\eta$ to have a lower bound to simplify the convergence bounds. We can also derive $\eta$-dependent bounds. Note that the ``large step size" phenomenon is also consistent with the single-agent setting (see Theorem 5.3 in~\citep{agarwal2020optimality} and discussion therein).
% }

% \yulai{
% To overcome the non-stationary environment faced by each player which is a
% major challenge in multi-agent learning, we adapt scheme Approximate Generalized Policy Iteration~\citep[Section 3.1]{perolat2015approximate} which combines Policy Iteration (PI) and Value Iteration (VI) under a generic formulation. In the \textbf{Iteration Step}, policy
% gradient updates are used for finding approximate $\inf_{f} \gT_{x,f}v $(cf. Section~\ref{sec:pre}) which is equivalent to letting $m \xrightarrow[]{} \infty$. In
% this view, our Algorithm is inspirit similar to PI. We remark that~\citep{ashutosh2020lower} provides lower bounds for PI on Multi-action MDP while we give upper convergence bounds for
% multi-player case.
% }

The proof of Theorem~\ref{thm:popu_NPG} further requires the following parts: mirror-descent type analysis of NPG used in~\citep{agarwal2020optimality} and simultaneous mirror descent for matrix games proposed in~\citep{rakhlin2013optimization}.
The full proof is deferred to Appendix~\ref{popu_thm_pf}.

\subsection{Extension: Entropy regularization}
 \label{sec_ext_entropy}
 Following~\citep{cen2020fast}, we give an extension of entropy-regularized NPG in the \textbf{Iteration Step} for Algorithm~\ref{alg:popu_NPG}.
 Denote $\tau$ as the regularization term, the entropy regularized value function is formulated as
 \begin{align}
 V_\tau^{x, f}(\sigma) = V^{x, f}(\sigma) - \tau \gH(\sigma, f)
 \end{align}
 where $\gH(\sigma, f) = \frac{1}{1-\gamma} \E_{s \sim d_\sigma^{x, f}} \E_{b \sim f} \log{\frac{1}{f(b|s)}}$ is the entropy term w.r.t. min player $f$.
 Note that $V_\tau^{x, f}(s) \in \big{[} -\tau \log{|\gA|}, 1 \big{]}, \forall s \in \gS$.
 
 Entropy regularization requires us to minimize $V_\tau$ instead of original value function $V$. Denote $V_\tau^*(\sigma) = \min_{f} V_\tau^{x, f}(\sigma) = V^{x, f_\tau^*}(\sigma)- \tau \gH(\sigma, f_\tau^*)$, the following sandwich bound holds
 \begin{align*}
 V^{x, f_\tau^*}(\sigma) &\ge V^{x, f^*(x)}(\sigma) \ge V_\tau^{x, f^*(x)}(\sigma)\\
 &\ge V_\tau^*(\sigma) \ge V^{x, f_\tau^*}(\sigma) - \frac{\tau}{1-\gamma} \log{}|\gA|.
 \end{align*}
 Therefore, the regularized problem and the original problem are close when $\tau$ is small.
 
 \textbf{NPG methods with entropy regularization.} Let $\eta = \frac{1-\gamma}{\tau},$ we have the NPG update rule
 \begin{align*}
 &\theta^{t+1} \xleftarrow[]{} \theta^t - \eta F_\sigma(\theta^t)^\dag \nabla_{\theta} V_\tau^{x, f^t}(\sigma),\\
 &f^{t+1}(b|s) \propto \exp{\left(-\frac{1}{\tau} \sum_a x(a|s) Q_\tau^{x, f^t}(s, a, b)\right)}. 
 \end{align*}
%  We aim to bound $V^{x, f^T}(\sigma) -V^{x, f^*(x)}(\sigma)$ through optimizing $V_\tau$, denote $V_\tau^* = V^{x, f^*(x)}$ for short, observe that
%  \begin{align*}
%  &\quad V^{x, f^T}(\sigma) - V^{x, f^*(x)}(\sigma)\\
%  &= V^{x, f^T}(\sigma) - V_\tau^{x, f^T}(\sigma) + V_\tau^{x, f^T}(\sigma) - V_\tau^*(\sigma) \\
%  &\le \frac{\tau}{1-\gamma} \log{|\gA|} + V_\tau^{x, f^T}(\sigma) - V_\tau^*(\sigma) + 0.
%  \end{align*}
 The following theorem shows the performance improvement over Theorem~\ref{thm:popu_NPG}.
 \begin{theorem} \label{thm:popu_NPG_entropy}
 For entropy regularized Algorithm~\ref{alg:popu_NPG}, after $K$ outer loops, one-sided measure $V^*(\rho) - \inf_{f} V^{x^K, f}(\rho)$ is bounded by 
 	\begin{align*}
  \widetilde{O} \left( \frac{ \gamma^ T \gC_{\rho, \sigma}^{1, K, 0}}{(1-\gamma)^2} \left\| \frac{\sigma}{\mu_\tau^*} \right\|_\infty + \frac{ \gC_{\rho, \sigma}^{0, K, 0} \log{T^\prime}}{(1-\gamma)^4 T^\prime} + \frac{ \gamma^K \gC_{\rho, \sigma}^{K, K+1, 0}}{1-\gamma} \right).
 \end{align*}
 \end{theorem}
 Here, $\mu_\tau^*$ is the one-sided \emph{stationary distribution} w.r.t. $x$ which satisfies: $\mu_\tau^* = d_{\mu_\tau^*}^{x, f_\tau^*}$. This argument indicates that the state visitation distribution remains unchanged if the initial state is already in a steady state. See Appendix~\ref{sec:pf_popu_entropy} for the full proof.
 
%  \yulai{
 Recently, ~\citet{perolat2021poincare} studied learning algorithms for extensive-form zero-sum games and they also used entropy regularization. Compared with their work, the differences include 1) we use policy optimization instead of value-based methods used in their paper. 2) our entropy regularization is a simple extension whereas the entropy regularization is crucial in~\citep{perolat2021poincare}: the regularization term gives strong convergence guarantees in monotone games.
%  }

\section{Online Algorithm with Function Approximation} 
\label{sec:online_algo}
In this section, we extend Algorithm~\ref{alg:popu_NPG} to the realistic online setting with function approximation, in which the parameterization we adopt is defined in Section~\ref{sec_FA}.
In this setting, we only observe samples (instead of the population quantities in Section~\ref{sec:population_algo}).
The pseudo-code is listed in Algorithm~\ref{alg:online_NPG}.

To obtain estimates of quantities, we adopt the episodic sampling oracle (cf. Section~\ref{sec:pre}) to provide transition tuples for estimating $A_s(a,b)$ in \textbf{Greedy Step} (cf. Eq.~\ref{eq_greedy_step}). This sampling oracle is also used in the \textbf{Iteration Step} to estimate value functions and gradients.
See Appendix~\ref{sec:pf_online} for more details about how we use the sampling oracle.
% \simon{maybe delete this paragraph.}

In Section~\ref{sec:rel}, we have pointed out that our algorithm has a smaller sample complexity of $O\left(\epsilon^{-6}\right)$ comparing to $O\left(\epsilon^{-12.5}\right)$~\citep{daskalakis2020independent}. As for the computational complexity, we remark that we only need projecting onto an $L_2$-norm ball in Algorithm~\ref{alg:online_NPG},~\ref{alg:online_MD}, which has the same time complexity as computing the gradient (linear in the dimension of $\theta$), so our algorithms are computationally efficient.

Now we describe our algorithm.
Specifically, we let $\xi$ and $\theta$ be parameters of $x$ and $f$, respectively. \footnote{We assume that the two players share the same parameter set only for clarity. We only need some minor modifications in the analysis to extend our results to the setting where two opposing players have different capabilities. Specifically, we only need to treat $W$ (norm-bound of updates), $D$ (regularity condition on features), and $\eta$ separately for each agent.}
% pseudo-code is listed in Algorithm~\ref{alg:online_NPG}. 
 The output and motivation of the \textbf{Greedy Step} and the \textbf{Iteration Step} are analogous to those in Algorithm~\ref{alg:popu_NPG}. 
%  However one could only samples from the sampling oracle to estimate required quantities (gradients, value function, etc), instead of using the exact population quantities. 
 In both steps, we need to take sample-based NPG updates which are forced to be constrained in a convex set $\mathcal{W} = \{w: \|w\|_2 \le W\}$ for analysis. From now, we denote $W$ as the bound of this norm-constrained convex set where each NPG update lies.

Again, we first discuss the \textbf{Greedy Step} whose pseudo-code is listed in Algorithm~\ref{alg:online_MD}.
%Algorithm~\ref{alg:online_MD} is the subroutine of \textbf{Greedy Step}.
Our goal is still to obtain a near-optimal $x^k$ with respect to $\gT V_{k-1}$.
Algorithm~\ref{alg:online_MD} is similar to Algorithm~\ref{alg:popu_MD} in spirit.
The main difference is that we use a sample-based NPG update rule for both $x$ and $f$.
Ideally, we wish to find simultaneous updates $w_f^*$ and $w_x^*$ for the players. Both are minimizers of quadratic loss
\vspace{-1em}
\begin{align*}
    w_f^* &= \arg \min_w
    \E_{s\sim \sigma}\E_{b \sim f^t(\cdot|s)}\\
    &\left( w^\top \nabla_\theta \log f_\theta^t(b|s)\! - \!\left[ (x_t^\top A_s)_b - \phi_s(f_t, x_t) \right] \right)^2.\\
    w_x^* &= \arg \min_w \E_{s \sim \sigma}\E_{a \sim x^t(\cdot|s)}\\
    &\left( w^\top \nabla_\xi \log x_\xi^t(a|s)\! - \!\left[ (A_s f_t)_a - \phi_s(f_t, x_t) \right] \right)^2.
\end{align*}
Then the updates take the form
$    \theta_{t+1} = \theta_t - \eta w_f^*$ and  $\quad\xi_{t+1} = \xi_t + \eta w_x^*.$
Along the way, the sampling oracle is used to approximate $w_f^*, w_x^*$. After $T^\prime$ iterations, we are able to output an approximate solution $\Bar{x_{T^\prime}}$ by averaging $\{x_t\}_{t=1}^{T'}$.

After obtaining $x^k$ from the \textbf{Greedy Step}, we adapt NPG updates (Eq.~\ref{eq_popu_fisher_update}) in Algorithm~\ref{alg:popu_NPG} to the online setting.
 
Denote $\nu^t=\nu_{\nu_0}^{x, f^t}$ for simplicity. Ideally, NPG update in the \textbf{Iteration Step} takes the form 
\begin{gather}
     w^t \in \arg \min \mathop{\E} \limits_{ \substack{s,a,b \sim \nu^t}} \!\left( w^\top\nabla_{\theta} \log{f(b|s)} \!-\! A^{x, f}(s,a,b) \right)^2, \notag\\
     \theta^{t+1} = \theta^t - \eta w^t. \label{sec5_update_rule}
\end{gather}
We perform sample-based quadratic loss minimization, which shares similarity with the former step: it takes $N$ steps of \emph{projected gradient descent} to return an approximate update.

Now we state our main theorem.
\setlength{\textfloatsep}{0.1cm}
\setlength{\floatsep}{0.1cm}
\begin{algorithm}[t]
\begin{algorithmic}
  \caption{Online Two-Player NPG}  
  \label{alg:online_NPG}  
    \REQUIRE  
          $V_0= 0$ value function
        \ENSURE  
          Approximate policy $x^K$ at NE
        \FOR{$k=1,2,\cdots,K$}
          \STATE \textbf{Greedy Step:}
           \STATE Run Algorithm~\ref{alg:online_MD} returns $x^k$ with $T^\prime$ iterations.  
    \STATE \textbf{Iteration Step:}
    \STATE Fix $x= x^k$, initialize $\theta^{(0)} = 0$.
    \FOR{$t=0,1, \cdots,T-1 $}
       \STATE Initialize $w_0=0$.
       \FOR{$n = 0,1, \cdots, N-1 $}
           \STATE Sample $s, a, b \sim \nu^t$, then obtain $\hat{Q}(s, a, b)$ using the sampling oracle.
           \STATE Sample $b^\prime \sim f^{t} (\cdot|s)$, observe:
           \STATE $g_n\! =\! \hat{Q}(\!s,\! a,\! b)\! \left(\nabla_{\!\theta}\! \log{f^{t}(b|s) }\! -\! \nabla_{\!\theta} \! \log{f^{t}(b^\prime|s)} \right)$.
           \vspace{-1em}
           \begin{align*}
               w_{n+1} &= \text{Proj}_{\gW} \big{[} w_n - \\
               &\! 2\alpha\! \left(w_n^\top \nabla_{\!\theta}\! \log{f^{t}(b|s)} \nabla_{\!\theta}\! \log{f^{t}(b|s)}\! -\! g_n \right) \big{]}.
           \end{align*} \vspace{-2em}
       \ENDFOR
       \STATE Set $\hat{w}^{t} = \frac{1}{N} \sum_{n=1}^N w_n$.
       \STATE Update $\theta^{(t+1)} = \theta^{(t)} - \eta \hat{w}^{t}$.
    \ENDFOR
    \STATE Randomly sample $f$ from $f^{t} (t= 0,1\cdots T-1$). 
    \STATE Denote $V_{k}$ for $V^{x, f}$.
    \ENDFOR
\end{algorithmic}  
\end{algorithm} 

\begin{algorithm}[t]
\begin{algorithmic}
    \caption{Online Greedy Step with Function-Approx}
    \label{alg:online_MD}
    \REQUIRE $\theta_1, \xi_1 = \textbf{0} \in \sR^d$
    \ENSURE $\Bar{x_{T^\prime}}$ as average of $\{x_t\}, t \in [T^\prime]$
    \FOR{$t=1,2,\cdots,T^\prime$}
        \STATE \textbf{min player:}
        Initialize $w_0 = 0$.
        \FOR{$n = 0, 1, 2\cdots N^\prime-1$}
         \STATE Sample $s \sim \sigma(s), a \sim x^t(\cdot|s), b \sim f^t(\cdot|s), s^\prime \sim \gP(\cdot|s,a,b), b^\prime \sim f^t(\cdot|s)$, observe:\\
         $g_n = [r(s,a,b)+\gamma V_{k-1}(s^\prime)]\cdot (\nabla_\theta \log{f^t(b|s)} - \nabla_\theta \log{f^t(b^\prime|s)})$.
\STATE Update: 
             $w_{n+1} = \text{Proj}_\gW [w_n - 2\alpha^\prime \cdot (w_n^\top \nabla_\theta \log{f^t(b|s)} \nabla_\theta \log{f^t(b|s)} - g_n)].$
        \ENDFOR
        \STATE $\hat{w}^t = \frac{1}{N^\prime} \sum_{n=1}^{N^\prime}w_n$.
        \STATE Update: $\theta_{t+1} = \theta_t - \eta^\prime \hat{w}^t$.
        \STATE \textbf{max player:}
        Initialize $w_0 = 0$.
        \FOR{$n = 0, 1, 2\cdots N^\prime-1$}
         \STATE Sample $s \sim \sigma(s), a \sim x^t(\cdot|s), b \sim f^t(\cdot|s), s^\prime \sim \gP(\cdot|s,a,b), a^\prime \sim x^t(\cdot|s)$, observe:\\
         $g_n = [r(s,a,b)+\gamma V_{k-1}(s^\prime)]\cdot (\nabla_\xi \log{x^t(a|s)} - \nabla_\xi \log{x^t(a^\prime|s)})$.
         \STATE Update: $
             w_{n+1} = \text{Proj}_\gW \Big{[}w_n - 2\alpha^\prime \cdot(w_n^\top \nabla_\xi \log{x^t(a|s)} \nabla_\xi \log{x^t(a|s)} - g_n)\Big{]}.
         $
        \ENDFOR
        \STATE $\hat{w}^t = \frac{1}{N^\prime}\sum_{n=1}^{N^\prime}w_n$.
        \STATE Update: $\xi_{t+1} = \xi_t + \eta^\prime \hat{w}^t$.
    \ENDFOR
\end{algorithmic}
\end{algorithm}

\begin{theorem} \label{thm:online_NPG}
For Algorithm~\ref{alg:online_NPG},
suppose in the \textbf{Greedy Step}: 
$\forall t \in [T^\prime -1], \inf_{s,a} x^t(a|s), \inf_{s,b} f^t(b|s) \ge \iota^2.$ Let $G = 4D( 2DW +\frac{2}{1-\gamma})$. Set $\eta=\sqrt{\frac{2\log{|\gA|}}{D^2 W^2 T}}, \eta^\prime=\sqrt{\frac{2\log{|\gA|}}{D^2 W^2 T^\prime}}, \alpha = \frac{W}{G\sqrt{N}}, \alpha^\prime = \frac{W}{G\sqrt{N^\prime}}$. After K outer loops, $\E \left[V^*(\rho) - \inf_{f} V^{x^K, f}(\rho) \right]$ is bounded by
    \begin{align*}
 \widetilde{\gO} \left( \frac{\gC_{\rho, \sigma}^{1, K, 0}}{(1-\gamma)^2} \epsilon + \frac{\gC_{\rho, \sigma}^{0, K, 0} }{(1-\gamma)^2} \epsilon^\prime + \frac{\gamma^K}{ 1-\gamma } \gC_{\rho, \sigma}^{K, K+1, 0}\right)
    \end{align*}
    where error terms $\epsilon, \epsilon^\prime$ are defined as
    \begin{align*}
        \epsilon &= \sqrt{\frac{\log{|\gA|}D^2 W^2}{T}} + \frac{|\gA|}{(1-\gamma)^2} \sqrt{\gC_{\sigma, \sigma}^\prime \frac{GW}{\sqrt{N}}} +\\
        &\quad \frac{|\gA|}{(1-\gamma)^2} \sqrt{\gC_{\sigma, \sigma}^\prime \cdot \epsilon_{approx} }\\
        \epsilon^\prime &= \sqrt{\frac{\log|\gA| D^2 W^2 }{T^\prime}} + \iota\left(\sqrt{\epsilon_{approx}^\prime} + \frac{\sqrt{G W}}{{N^\prime}^{\frac{1}{4}}}\right)
    \end{align*}
\end{theorem}
% \simon{What's $C'$?, change $\beta$, What's $W$?} \yulai{C' is concentrability coefficient, see Definition 1.}\yulai{$\beta$ changed.} \yulai{$W$ is norm bound parametric set of each NPG update. Restate for clarity in line 278.}
\vspace{-1em}
Here $\epsilon_{approx}$ and $\epsilon_{approx}^\prime$ are \emph{approximation errors} coming from \textbf{Greedy} and \textbf{Iteration Steps} (cf. Definition~\ref{defn:approx_error}). We remind that $D$ is a regularity condition on features, with which we could show log-linear parameterization is $D^2$-smooth (cf. Section~\ref{sec_FA}). See Appendix~\ref{sec:pf_online} for specific expressions.

Similarly, the exponential $\gamma^K$ in Theorem~\ref{thm:online_NPG} implies that we only need a few outer iterations. When considering concentrability coefficients as constants, the dependency on $T$ is a slower $T^{-\nicefrac{1}{2}}$ rate while the sampling efficiency takes a $N^{-\nicefrac{1}{4}}$ rate. Both match the rates in the sampling-based single-agent NPG analysis~\citep{agarwal2020optimality}. We note that iteration counts $T, T^\prime$ and sample counts $N, N^\prime$ have the same exponent. There is no explicit dependence on state-space $\gS$ in the theorem, hence our online algorithm proves nice guarantees for function approximation even in the infinite-state setting. Instead, the bounds have parametric representation-related terms: $D$ upper bounds feature norms $\|\phi_{s,a}\|$ and $W$ restricts each NPG update. The term $\iota$ bounds two policy probabilities from below and it must be greater than $0$ since we adopt log-linear parameterization. 
In spirit, $\iota$ is similar to concentrability coefficients which reflect the inherent dynamics of Markov games.

% \yulai{
In the worst case, the concentrability coefficient scales as large as the number of states, and the bounds for function approximation are only meaningful in the benign case where the concentrability coefficient is small. However, we note that that, the dependency on concentrability is unavoidable: A hard example for the single-agent setting was given in~\citep{agarwal2020optimality}. Since our Markov-Game (MG) setting is a generalization of the single-agent setting, their hard example also applies to our setting. Moreover, we argue that  the coefficients can be small when there are some restrictions in the dynamics (see discussions in~\citep{chen2019information}). We also use the same definition as in the prior work value-based learning~\citep{perolat2015approximate}. The recent work~\citep{daskalakis2020independent} also assumed this structure of MGs to analyze policy-based methods.
% }
% }

% To prove Theorem~\ref{thm:online_NPG}, we follow our proof structure of Theorem~\ref{thm:popu_NPG} and add additional statistical and convergence analysis of the sampling and projected SGD. 
% See Appendix~\ref{online_thm_pf} for details.

\section{Conclusion}
\label{sec:conclusion}
This paper gave the first quantitative analysis of policy gradient methods for general two-player zero-sum Markov games with function approximation.
We quantified the performance gap of the output policy in terms of the number of iterations, number of samples, concentrability coefficients, and approximation error. 
% We believe our bounds are not tight. We leave it as a future direction to design optimal policy optimization methods for the setting considered in this paper. 
An interesting direction is to extend our results to more advanced PG methods such as PPO~\citep{schulman2017proximal}.
% \yulai{
% For the online setting with function approximation considered in this paper, we leave it as an important future direction to prove lower bounds.} \yulai{Lastly, we focus on the setting where the policy is linear whereas ~\citet{xie2020learning} and ~\citet{chen2021almost} put assumptions on the transition model. We believe it is very interesting to design provably efficient policy-based algorithms for their settings
% }

\subsubsection*{Acknowledgements}
JDL acknowledges support of the ARO under MURI Award W911NF-11-1-0304,  the Sloan Research Fellowship, NSF CCF 2002272, NSF IIS 2107304,  and an ONR Young Investigator Award.
SSD acknowledges funding from NSF Award’s IIS-2110170 and DMS-2134106.

\bibliographystyle{plainnat}
\bibliography{references}

\begin{thebibliography}{59}
\providecommand{\natexlab}[1]{#1}
\providecommand{\url}[1]{\texttt{#1}}
\expandafter\ifx\csname urlstyle\endcsname\relax
  \providecommand{\doi}[1]{doi: #1}\else
  \providecommand{\doi}{doi: \begingroup \urlstyle{rm}\Url}\fi

\bibitem[Agarwal et~al.(2020)Agarwal, Kakade, Lee, and
  Mahajan]{agarwal2020optimality}
Alekh Agarwal, Sham~M Kakade, Jason~D Lee, and Gaurav Mahajan.
\newblock Optimality and approximation with policy gradient methods in {M}arkov
  decision processes.
\newblock In \emph{Conference on Learning Theory}, pages 64--66. PMLR, 2020.

\bibitem[Antos et~al.(2008)Antos, Szepesv{\'a}ri, and Munos]{antos2008learning}
Andr{\'a}s Antos, Csaba Szepesv{\'a}ri, and R{\'e}mi Munos.
\newblock Learning near-optimal policies with {B}ellman-residual minimization
  based fitted policy iteration and a single sample path.
\newblock \emph{Machine Learning}, 71\penalty0 (1):\penalty0 89--129, 2008.

\bibitem[Babichenko and Rubinstein(2020)]{babichenko2020communication}
Yakov Babichenko and Aviad Rubinstein.
\newblock Communication complexity of approximate {N}ash equilibria.
\newblock \emph{Games and Economic Behavior}, 2020.

\bibitem[Bai and Jin(2020)]{bai2020provable}
Yu~Bai and Chi Jin.
\newblock Provable self-play algorithms for competitive reinforcement learning.
\newblock In \emph{International Conference on Machine Learning}, pages
  551--560. PMLR, 2020.

\bibitem[Bai et~al.(2020)Bai, Jin, and Yu]{bai2020near}
Yu~Bai, Chi Jin, and Tiancheng Yu.
\newblock Near-optimal reinforcement learning with self-play.
\newblock \emph{arXiv preprint arXiv:2006.12007}, 2020.

\bibitem[Bellemare et~al.(2019)Bellemare, Dabney, Dadashi, Taiga, Castro,
  Le~Roux, Schuurmans, Lattimore, and Lyle]{bellemare2019geometric}
Marc Bellemare, Will Dabney, Robert Dadashi, Adrien~Ali Taiga, Pablo~Samuel
  Castro, Nicolas Le~Roux, Dale Schuurmans, Tor Lattimore, and Clare Lyle.
\newblock A geometric perspective on optimal representations for reinforcement
  learning.
\newblock In \emph{Advances in Neural Information Processing Systems}, pages
  4358--4369, 2019.

\bibitem[Bhandari and Russo(2019)]{bhandari2019global}
Jalaj Bhandari and Daniel Russo.
\newblock Global optimality guarantees for policy gradient methods.
\newblock \emph{arXiv preprint arXiv:1906.01786}, 2019.

\bibitem[Branavan et~al.(2009)Branavan, Chen, Zettlemoyer, and
  Barzilay]{branavan2009reinforcement}
S.~R.~K. Branavan, Harr Chen, Luke~S. Zettlemoyer, and Regina Barzilay.
\newblock Reinforcement learning for mapping instructions to actions.
\newblock In \emph{Proceedings of the Joint Conference of the 47th Annual
  Meeting of the ACL and the 4th International Joint Conference on Natural
  Language Processing of the AFNLP: Volume 1 - Volume 1}, ACL '09, page
  82–90, USA, 2009. Association for Computational Linguistics.

\bibitem[Brown(1951)]{brown1951iterative}
George~W Brown.
\newblock Iterative solution of games by fictitious play.
\newblock \emph{Activity analysis of production and allocation}, 13\penalty0
  (1):\penalty0 374--376, 1951.

\bibitem[Brown and Sandholm(2018)]{brown2018superhuman}
Noam Brown and Tuomas Sandholm.
\newblock Superhuman {AI} for heads-up no-limit poker: Libratus beats top
  professionals.
\newblock \emph{Science}, 359\penalty0 (6374):\penalty0 418--424, 2018.

\bibitem[Brown and Sandholm(2019)]{brown2019solving}
Noam Brown and Tuomas Sandholm.
\newblock Solving imperfect-information games via discounted regret
  minimization.
\newblock In \emph{Proceedings of the AAAI Conference on Artificial
  Intelligence}, volume~33, pages 1829--1836, 2019.

\bibitem[Bu et~al.(2019)Bu, Ratliff, and Mesbahi]{bu2019global}
Jingjing Bu, Lillian~J Ratliff, and Mehran Mesbahi.
\newblock Global convergence of policy gradient for sequential zero-sum linear
  quadratic dynamic games.
\newblock \emph{arXiv preprint arXiv:1911.04672}, 2019.

\bibitem[Cen et~al.(2020)Cen, Cheng, Chen, Wei, and Chi]{cen2020fast}
Shicong Cen, Chen Cheng, Yuxin Chen, Yuting Wei, and Yuejie Chi.
\newblock Fast global convergence of natural policy gradient methods with
  entropy regularization.
\newblock \emph{arXiv preprint arXiv:2007.06558}, 2020.

\bibitem[Chen and Jiang(2019)]{chen2019information}
Jinglin Chen and Nan Jiang.
\newblock Information-theoretic considerations in batch reinforcement learning.
\newblock In \emph{International Conference on Machine Learning}, pages
  1042--1051. PMLR, 2019.

\bibitem[Daskalakis et~al.(2007)Daskalakis, Mehta, and
  Papadimitriou]{daskalakis2007progress}
Constantinos Daskalakis, Aranyak Mehta, and Christos Papadimitriou.
\newblock Progress in approximate {N}ash equilibria.
\newblock In \emph{Proceedings of the 8th ACM conference on Electronic
  commerce}, pages 355--358, 2007.

\bibitem[Daskalakis et~al.(2018)Daskalakis, Ilyas, Syrgkanis, and
  Zeng]{daskalakis2018training}
Constantinos Daskalakis, Andrew Ilyas, Vasilis Syrgkanis, and Haoyang Zeng.
\newblock Training {GAN}s with optimism.
\newblock In \emph{International Conference on Learning Representations}, 2018.
\newblock URL \url{https://openreview.net/forum?id=SJJySbbAZ}.

\bibitem[Daskalakis et~al.(2020)Daskalakis, Foster, and
  Golowich]{daskalakis2020independent}
Constantinos Daskalakis, Dylan~J Foster, and Noah Golowich.
\newblock Independent policy gradient methods for competitive reinforcement
  learning.
\newblock \emph{arXiv preprint arXiv:2101.04233}, 2020.

\bibitem[Deligkas et~al.(2017)Deligkas, Fearnley, Savani, and
  Spirakis]{deligkas2017computing}
Argyrios Deligkas, John Fearnley, Rahul Savani, and Paul Spirakis.
\newblock Computing approximate {N}ash equilibria in polymatrix games.
\newblock \emph{Algorithmica}, 77\penalty0 (2):\penalty0 487--514, 2017.

\bibitem[Filar and Vrieze(2012)]{filar2012competitive}
Jerzy Filar and Koos Vrieze.
\newblock \emph{Competitive {M}arkov decision processes}.
\newblock Springer Science \& Business Media, 2012.

\bibitem[Foerster et~al.(2017)Foerster, Chen, Al-Shedivat, Whiteson, Abbeel,
  and Mordatch]{foerster2017learning}
Jakob~N Foerster, Richard~Y Chen, Maruan Al-Shedivat, Shimon Whiteson, Pieter
  Abbeel, and Igor Mordatch.
\newblock Learning with opponent-learning awareness.
\newblock \emph{arXiv preprint arXiv:1709.04326}, 2017.

\bibitem[Gimpel and Smith(2010)]{gimpel2010softmax}
Kevin Gimpel and Noah~A Smith.
\newblock Softmax-margin crfs: Training log-linear models with cost functions.
\newblock In \emph{Human Language Technologies: The 2010 Annual Conference of
  the North American Chapter of the Association for Computational Linguistics},
  pages 733--736, 2010.

\bibitem[G{\"o}{\"o}s and Rubinstein(2018)]{goos2018near}
Mika G{\"o}{\"o}s and Aviad Rubinstein.
\newblock Near-optimal communication lower bounds for approximate {N}ash
  equilibria.
\newblock In \emph{2018 IEEE 59th Annual Symposium on Foundations of Computer
  Science (FOCS)}, pages 397--403. IEEE, 2018.

\bibitem[Guo et~al.(2016)Guo, Singh, Lewis, and Lee]{guo2016deep}
Xiaoxiao Guo, Satinder Singh, Richard Lewis, and Honglak Lee.
\newblock Deep learning for reward design to improve monte carlo tree search in
  atari games.
\newblock \emph{arXiv preprint arXiv:1604.07095}, 2016.

\bibitem[Heess et~al.(2013)Heess, Silver, and Teh]{heess2013actor}
Nicolas Heess, David Silver, and Yee~Whye Teh.
\newblock Actor-critic reinforcement learning with energy-based policies.
\newblock In \emph{European Workshop on Reinforcement Learning}, pages 45--58.
  PMLR, 2013.

\bibitem[Heinrich et~al.(2015)Heinrich, Lanctot, and
  Silver]{heinrich2015fictitious}
Johannes Heinrich, Marc Lanctot, and David Silver.
\newblock Fictitious self-play in extensive-form games.
\newblock In \emph{International Conference on Machine Learning}, pages
  805--813, 2015.

\bibitem[Jiang et~al.(2015)Jiang, Kulesza, and Singh]{jiang2015abstraction}
Nan Jiang, Alex Kulesza, and Satinder Singh.
\newblock Abstraction selection in model-based reinforcement learning.
\newblock In \emph{International Conference on Machine Learning}, pages
  179--188, 2015.

\bibitem[Kakade(2002)]{kakade2002natural}
Sham~M Kakade.
\newblock A natural policy gradient.
\newblock In \emph{Advances in neural information processing systems}, pages
  1531--1538, 2002.

\bibitem[Littman(1994)]{littman1994markov}
Michael~L Littman.
\newblock Markov games as a framework for multi-agent reinforcement learning.
\newblock In \emph{Machine learning proceedings 1994}, pages 157--163.
  Elsevier, 1994.

\bibitem[Lockhart et~al.(2019)Lockhart, Lanctot, P{\'e}rolat, Lespiau, Morrill,
  Timbers, and Tuyls]{lockhart2019computing}
Edward Lockhart, Marc Lanctot, Julien P{\'e}rolat, Jean-Baptiste Lespiau,
  Dustin Morrill, Finbarr Timbers, and Karl Tuyls.
\newblock Computing approximate equilibria in sequential adversarial games by
  exploitability descent.
\newblock \emph{arXiv preprint arXiv:1903.05614}, 2019.

\bibitem[Menache et~al.(2005)Menache, Mannor, and Shimkin]{menache2005basis}
Ishai Menache, Shie Mannor, and Nahum Shimkin.
\newblock Basis function adaptation in temporal difference reinforcement
  learning.
\newblock \emph{Annals of Operations Research}, 134\penalty0 (1):\penalty0
  215--238, 2005.

\bibitem[Mertikopoulos et~al.(2019)Mertikopoulos, Lecouat, Zenati, Foo,
  Chandrasekhar, and Piliouras]{mertikopoulos2018optimistic}
Panayotis Mertikopoulos, Bruno Lecouat, Houssam Zenati, Chuan-Sheng Foo, Vijay
  Chandrasekhar, and Georgios Piliouras.
\newblock Optimistic mirror descent in saddle-point problems: Going the
  extra(-gradient) mile.
\newblock In \emph{International Conference on Learning Representations}, 2019.
\newblock URL \url{https://openreview.net/forum?id=Bkg8jjC9KQ}.

\bibitem[Mousavi et~al.(2017)Mousavi, Schukat, and Howley]{mousavi2017traffic}
Seyed~Sajad Mousavi, Michael Schukat, and Enda Howley.
\newblock Traffic light control using deep policy-gradient and
  value-function-based reinforcement learning.
\newblock \emph{IET Intelligent Transport Systems}, 11\penalty0 (7):\penalty0
  417--423, 2017.

\bibitem[Munos(2005)]{munos2005error}
R{\'e}mi Munos.
\newblock Error bounds for approximate value iteration.
\newblock In \emph{Proceedings of the National Conference on Artificial
  Intelligence}, volume~20, page 1006. Menlo Park, CA; Cambridge, MA; London;
  AAAI Press; MIT Press; 1999, 2005.

\bibitem[Patek(1997)]{patek1997stochastic}
Stephen~David Patek.
\newblock \emph{Stochastic and shortest path games: theory and algorithms}.
\newblock PhD thesis, Massachusetts Institute of Technology, 1997.

\bibitem[Perolat et~al.(2015)Perolat, Scherrer, Piot, and
  Pietquin]{perolat2015approximate}
Julien Perolat, Bruno Scherrer, Bilal Piot, and Olivier Pietquin.
\newblock Approximate dynamic programming for two-player zero-sum {M}arkov
  games.
\newblock In \emph{International Conference on Machine Learning}, pages
  1321--1329, 2015.

\bibitem[P{\'e}rolat et~al.(2016)P{\'e}rolat, Piot, Scherrer, and
  Pietquin]{perolat2016use}
Julien P{\'e}rolat, Bilal Piot, Bruno Scherrer, and Olivier Pietquin.
\newblock On the use of non-stationary strategies for solving two-player
  zero-sum {M}arkov games.
\newblock In \emph{AISTATS}, pages 893--901, 2016.

\bibitem[Perolat et~al.(2021)Perolat, Munos, Lespiau, Omidshafiei, Rowland,
  Ortega, Burch, Anthony, Balduzzi, De~Vylder, et~al.]{perolat2021poincare}
Julien Perolat, Remi Munos, Jean-Baptiste Lespiau, Shayegan Omidshafiei, Mark
  Rowland, Pedro Ortega, Neil Burch, Thomas Anthony, David Balduzzi, Bart
  De~Vylder, et~al.
\newblock From poincar{\'e} recurrence to convergence in imperfect information
  games: Finding equilibrium via regularization.
\newblock In \emph{International Conference on Machine Learning}, pages
  8525--8535. PMLR, 2021.

\bibitem[Rakhlin and Sridharan(2013)]{rakhlin2013optimization}
Sasha Rakhlin and Karthik Sridharan.
\newblock Optimization, learning, and games with predictable sequences.
\newblock In \emph{Advances in Neural Information Processing Systems}, pages
  3066--3074, 2013.

\bibitem[Robinson(1951)]{robinson1951iterative}
Julia Robinson.
\newblock An iterative method of solving a game.
\newblock \emph{Annals of mathematics}, pages 296--301, 1951.

\bibitem[Scherrer(2014)]{scherrer2014approximate}
Bruno Scherrer.
\newblock Approximate policy iteration schemes: a comparison.
\newblock In \emph{International Conference on Machine Learning}, pages
  1314--1322, 2014.

\bibitem[Scherrer et~al.(2012)Scherrer, Gabillon, Ghavamzadeh, and
  Geist]{scherrer2012approximate}
Bruno Scherrer, Victor Gabillon, Mohammad Ghavamzadeh, and Matthieu Geist.
\newblock Approximate modified policy iteration.
\newblock \emph{arXiv preprint arXiv:1205.3054}, 2012.

\bibitem[Scherrer et~al.(2015)Scherrer, Ghavamzadeh, Gabillon, Lesner, and
  Geist]{scherrer2015approximate}
Bruno Scherrer, Mohammad Ghavamzadeh, Victor Gabillon, Boris Lesner, and
  Matthieu Geist.
\newblock Approximate modified policy iteration and its application to the game
  of tetris.
\newblock \emph{J. Mach. Learn. Res.}, 16:\penalty0 1629--1676, 2015.

\bibitem[Schulman et~al.(2015)Schulman, Levine, Abbeel, Jordan, and
  Moritz]{schulman2015trust}
John Schulman, Sergey Levine, Pieter Abbeel, Michael Jordan, and Philipp
  Moritz.
\newblock Trust region policy optimization.
\newblock In \emph{International conference on machine learning}, pages
  1889--1897, 2015.

\bibitem[Schulman et~al.(2017)Schulman, Wolski, Dhariwal, Radford, and
  Klimov]{schulman2017proximal}
John Schulman, Filip Wolski, Prafulla Dhariwal, Alec Radford, and Oleg Klimov.
\newblock Proximal policy optimization algorithms.
\newblock \emph{arXiv preprint arXiv:1707.06347}, 2017.

\bibitem[Shalev-Shwartz and Ben-David(2014)]{shalev2014understanding}
Shai Shalev-Shwartz and Shai Ben-David.
\newblock \emph{Understanding machine learning: From theory to algorithms}.
\newblock Cambridge university press, 2014.

\bibitem[Shani et~al.(2020)Shani, Efroni, and Mannor]{shani2020adaptive}
Lior Shani, Yonathan Efroni, and Shie Mannor.
\newblock Adaptive trust region policy optimization: Global convergence and
  faster rates for regularized {MDP}s.
\newblock In \emph{Proceedings of the AAAI Conference on Artificial
  Intelligence}, volume~34, pages 5668--5675, 2020.

\bibitem[Shapley(1953)]{shapley1953stochastic}
Lloyd~S Shapley.
\newblock Stochastic games.
\newblock \emph{Proceedings of the national academy of sciences}, 39\penalty0
  (10):\penalty0 1095--1100, 1953.

\bibitem[Silver et~al.(2014)Silver, Lever, Heess, Degris, Wierstra, and
  Riedmiller]{silver2014deterministic}
David Silver, Guy Lever, Nicolas Heess, Thomas Degris, Daan Wierstra, and
  Martin Riedmiller.
\newblock Deterministic policy gradient algorithms.
\newblock In \emph{International conference on machine learning}, pages
  387--395. PMLR, 2014.

\bibitem[Silver et~al.(2016)Silver, Huang, Maddison, Guez, Sifre, Van
  Den~Driessche, Schrittwieser, Antonoglou, Panneershelvam, Lanctot,
  et~al.]{silver2016mastering}
David Silver, Aja Huang, Chris~J Maddison, Arthur Guez, Laurent Sifre, George
  Van Den~Driessche, Julian Schrittwieser, Ioannis Antonoglou, Veda
  Panneershelvam, Marc Lanctot, et~al.
\newblock Mastering the game of {G}o with deep neural networks and tree search.
\newblock \emph{nature}, 529\penalty0 (7587):\penalty0 484--489, 2016.

\bibitem[Silver et~al.(2017)Silver, Schrittwieser, Simonyan, Antonoglou, Huang,
  Guez, Hubert, Baker, Lai, Bolton, et~al.]{silver2017mastering}
David Silver, Julian Schrittwieser, Karen Simonyan, Ioannis Antonoglou, Aja
  Huang, Arthur Guez, Thomas Hubert, Lucas Baker, Matthew Lai, Adrian Bolton,
  et~al.
\newblock Mastering the game of {G}o without human knowledge.
\newblock \emph{nature}, 550\penalty0 (7676):\penalty0 354--359, 2017.

\bibitem[Song et~al.(2018)Song, Ren, Sadigh, and Ermon]{song2018multi}
Jiaming Song, Hongyu Ren, Dorsa Sadigh, and Stefano Ermon.
\newblock Multi-agent generative adversarial imitation learning.
\newblock In \emph{Advances in neural information processing systems}, pages
  7461--7472, 2018.

\bibitem[Sutton et~al.(2000)Sutton, McAllester, Singh, and
  Mansour]{sutton2000policy}
Richard~S Sutton, David~A McAllester, Satinder~P Singh, and Yishay Mansour.
\newblock Policy gradient methods for reinforcement learning with function
  approximation.
\newblock In \emph{Advances in neural information processing systems}, pages
  1057--1063, 2000.

\bibitem[Tian et~al.(2019)Tian, Ma, Gong, Sengupta, Chen, Pinkerton, and
  Zitnick]{tian2019elf}
Yuandong Tian, Jerry Ma, Qucheng Gong, Shubho Sengupta, Zhuoyuan Chen, James
  Pinkerton, and Larry Zitnick.
\newblock Elf opengo: An analysis and open reimplementation of alphazero.
\newblock In \emph{International Conference on Machine Learning}, pages
  6244--6253. PMLR, 2019.

\bibitem[Vinyals et~al.(2019)Vinyals, Babuschkin, Czarnecki, Mathieu, Dudzik,
  Chung, Choi, Powell, Ewalds, Georgiev, et~al.]{vinyals2019grandmaster}
Oriol Vinyals, Igor Babuschkin, Wojciech~M Czarnecki, Micha{\"e}l Mathieu,
  Andrew Dudzik, Junyoung Chung, David~H Choi, Richard Powell, Timo Ewalds,
  Petko Georgiev, et~al.
\newblock Grandmaster level in starcraft ii using multi-agent reinforcement
  learning.
\newblock \emph{Nature}, 575\penalty0 (7782):\penalty0 350--354, 2019.

\bibitem[Weller et~al.(1998)Weller, Tabor, Jasak, and
  Fureby]{weller1998tensorial}
Henry~G Weller, Gavin Tabor, Hrvoje Jasak, and Christer Fureby.
\newblock A tensorial approach to computational continuum mechanics using
  object-oriented techniques.
\newblock \emph{Computers in physics}, 12\penalty0 (6):\penalty0 620--631,
  1998.

\bibitem[Yu et~al.(2019)Yu, Yang, Wang, and Wang]{yu2019provable}
Ming Yu, Zhuoran Yang, Mengdi Wang, and Zhaoran Wang.
\newblock Provable q-iteration with l infinity guarantees and function
  approximation.
\newblock In \emph{Workshop on Optimization and RL, NeurIPS}, 2019.

\bibitem[Zhang et~al.(2019)Zhang, Yang, and Basar]{zhang2019policy}
Kaiqing Zhang, Zhuoran Yang, and Tamer Basar.
\newblock Policy optimization provably converges to {N}ash equilibria in
  zero-sum linear quadratic games.
\newblock In \emph{Advances in Neural Information Processing Systems}, pages
  11602--11614, 2019.

\bibitem[Zhang et~al.(2020)Zhang, Koppel, Zhu, and Basar]{zhang2020global}
Kaiqing Zhang, Alec Koppel, Hao Zhu, and Tamer Basar.
\newblock Global convergence of policy gradient methods to (almost) locally
  optimal policies.
\newblock \emph{SIAM Journal on Control and Optimization}, 58\penalty0
  (6):\penalty0 3586--3612, 2020.

\bibitem[Zinkevich et~al.(2008)Zinkevich, Johanson, Bowling, and
  Piccione]{zinkevich2008regret}
Martin Zinkevich, Michael Johanson, Michael Bowling, and Carmelo Piccione.
\newblock Regret minimization in games with incomplete information.
\newblock In \emph{Advances in neural information processing systems}, pages
  1729--1736, 2008.

\end{thebibliography}

\onecolumn	
	\appendix
	
    \section{Basic results}
        In this section we provide some fundamental results for two-player zero-sum games and policy gradients. 

\begin{lemma}[contraction of Bellman operator]\label{contraction_bellman} We show $\gT_{x} v = \inf_{f} \gT_{x, f} v$ is a $\gamma$ contractor to $V^{x}$. Other forms of Bellman operators defined in Section~\ref{sec:pre} could be shown to hold contraction property with similar lines. 
\end{lemma}
\begin{proof}
 First we show $V^{x}$ is the unique fix point of $\gT_{x}$, this is because:
 \begin{align*}
     V^{x}(s) &= r (s, x(s), f^*(s)) + \gamma \sum_{s^\prime} \gP(s^\prime | s, x(s), f^*(s))V^{x}(s^\prime)\\
     &= \inf_{f} r (s, x(s), f(s)) + \gamma \sum_{s^\prime} \gP(s^\prime | s, x(s), f(s))V^{x}(s^\prime)\\
     &= \gT_{x} V^{x} (s)
 \end{align*}

Then for all function $v: \sR^{|\gS|} \to \sR^{|\gS|}$, 
\begin{align*}
    & \quad \left| \gT_{x}v (s) - \gT_{x} V^{x} (s)\right|\\
    &= \left| \inf_{f} \gT_{x, f}v (s) - \inf_{f} \gT_{x, f} V^{x}(s)\right|\\
    &= \max \left\{ \inf_{f} \gT_{x, f}v (s) - \inf_{f} \gT_{x, f} V^{x}(s), \inf_{f} \gT_{x, f} V^{x}(s) - \inf_{f} \gT_{x, f}v (s) \right\}
\end{align*}
Note that the first term could be upper bounded by
\begin{align*}
    & \quad \inf_{f} \gT_{x, f}v (s) - \inf_{f} \gT_{x, f} V^{x}(s)\\
    &= \inf_{f} \gT_{x, f}v (s) - \gT_{x, f^*} V^{x}(s)\\
    &\le \gT_{x, f^*}v (s) - \gT_{x, f^*} V^{x}(s)\\
    &= \gamma \sum_{s^\prime} \gP(s^\prime | s, x(s), f^*(s)) \left( v (s) - V^{x}(s)\right) \\
    &\le \gamma \left| v(s) - V^{x}(s) \right|
\end{align*}
The second term could be upper bounded similarly, hence we have
\begin{align*}
    \left| \gT_{x}v - \gT_{x} V^{x}\right| \le \gamma | v - V^{x}| 
\end{align*}
A direct application of contraction is $(\gT_{x})^\infty v = V^{x}$, which inspires the classical \emph{Value Iteration} algorithm~\citep{shapley1953stochastic}.
\end{proof}

To analyze our NPG algorithm based on approximate dynamic programming scheme, we introduce the following lemma to upper bounding global performance, which is very useful in other sections.
\begin{lemma} \label{original_form}
Let $\rho$ and $\sigma$ be distributions over states. With Algorithm~\ref{alg:popu_NPG}, after k iterations 
    \begin{align}
        V^*(\rho) - \inf_{f} V^{x^k, f}(\rho) &\le \frac{2(\gamma-\gamma^k) \gC_{\rho, \sigma}^{1, k, 0}}{(1-\gamma)^2}\epsilon + \frac{(1-\gamma^k)\gC_{\rho, \sigma}^{0, k, 0}}{(1-\gamma)^2}  \epsilon^\prime
        + \frac{ 2 \gamma^k \gC_{\rho, \sigma}^{k, k+1, 0}}{1-\gamma} ,
    \end{align}
    where
    \begin{gather*}
        \epsilon = \sup_{1 \le j \le k-1} \| \epsilon_j\|_{1, \sigma},\\
        \epsilon^\prime = \sup_{1 \le j \le k } \| \epsilon_j^\prime \|_{1,\sigma}.
    \end{gather*}
This Lemma could be directly extended to expectation form in Section~\ref{sec:pf_online}.
\end{lemma}

This is a straightforward application of the following theorem.\\
\textbf{\citep[Theorem 1]{perolat2015approximate}}
    Let $\rho$ and $\sigma$ be distributions over states. Let p, q and $q^\prime$ be such that $\frac{1}{q} + \frac{1}{q^\prime} = 1$. Approximate Generalized Policy Iteration takes the following update: 
    
    \begin{gather}
        \gT V_{k-1} \le \gT_{x^k} V_{k-1} + \epsilon_k^\prime \label{GPI_greedy}\\
            V_k = \left( \gT_{x^k} \right)^m V_{k-1} + \epsilon_k \label{GPI_iter}
    \end{gather}

Then, after k iterations, we have:
\begin{align*}
    \|l_k\|_{p, \rho} &\le \frac{2(\gamma-\gamma^k) (\gC_{q}^{1, k, 0})^{\frac{1}{p}} }{(1-\gamma)^2} 
    \sup_{1 \le j \le k-1} \| \epsilon_j\|_{pq^\prime, \sigma},\\
    & + \frac{(1-\gamma^k) (\gC_{q}^{0, k, 0})^{\frac{1}{p}}}{(1-\gamma)^2}
    \sup_{1 \le j \le k} \| \epsilon_j^\prime\|_{pq^\prime, \sigma},\\
    &+ \frac{ 2 \gamma^k}{1-\gamma} (\gC_{q}^{k, k+1, 0})^{\frac{1}{p}} \min(\|d_0\|_{pq^\prime, \sigma}, \|b_0\|_{pq^\prime, \sigma}).
\end{align*}
where
\begin{align*}
   \gC_{q}^{l, k, d} &= \frac{(1-\gamma)^2}{\gamma^l-\gamma^k} \sum_{i=l}^{k-1} \sum_{j=i}^\infty   c_q(j+d)\\
    l_k &= V^*- \inf_{f} V^{x^k, f} \\
     b_k &= V_k - \gT_{x^{k+1}} V_k.
\end{align*} 
Note the generalized norm of \emph{Radon-Nikodym} derivative is:
\begin{equation*}
    c_q(j) = \sup_{\mu_1, \nu_1, \cdots, \mu_j, \nu_j} \left\| \frac{ d(\rho \gP_{\mu_1,\nu_1}\cdots \gP_{\mu_j,\nu_j}) }{d \sigma} \right\|_{q, \sigma}
\end{equation*}

 Now we make adaptation to this theorem.
\begin{proof}
Set norm order $p=1$, then let $q \to \infty, q^\prime=1$. Note that, in reinforcement learning, $\rho$ has an explicit meaning of measure distribution or distribution for testing, while $\sigma$ stands for exploration distribution. Normally exploration should cover more states, e.g., $\sigma$ is a uniform distribution over all actions. \\
Now we provide detailed calculations with these parameter settings.
\begin{align*}
    c_{q \to \infty }(j) &= 
    \sup_{x^1, f^1, \cdots x^j, f^j} \left\| \frac{\rho \gP_{x^1, f^1}  \cdots \gP_{x^j, f^j} }{ \sigma} \right\|_{q \to \infty, \sigma}\\
    &= \lim_{q \to \infty} \left( \sum_s \sigma(s) \left\| \frac{\rho \gP_{x^1, f^1}  \cdots \gP_{x^j, f^j} (s) }{ \sigma(s)} \right\|^q \right)^{\frac{1}{q}}\\
    &= \sup_{x^1, f^1, \cdots x^j, f^j} \left\| \frac{\rho \gP_{x^1, f^1}  \cdots \gP_{x^j, f^j} }{ \sigma} \right\|_{\infty}\\
    &= c_{\rho, \sigma}(j)
\end{align*}
As for weighted norm $\sigma$, it holds:
\begin{align*}
    \|l_k\|_{1, \rho} &= \sum_s \rho(s) (V^*(s) - \inf_{f} V^{x^k, f}(s) ) \\
                      &= V^*(\rho) - \inf_{f} V^{x^k, f}(\rho)
\end{align*}

Notice in practice $V_0 (s)$ is initialized to be 0, hence
\begin{align*}
    \| b_0\|_{1, \sigma} &= \sum_s \sigma(s) |b_0(s)|\\
                         &= \sum_s \sigma(s) | V_0(s) - \gT_{x^1} V_0(s) | \\
                         &\le \sup_{s,a,b} r(s,a,b)\\
                         &\le 1
\end{align*}
 which gives that $\min(\|l_0\|_{1,\sigma}, \|b_0\|_{1, \sigma}) \le 1$. Then proof is completed via substitution.
\end{proof}

 \begin{lemma}[Policy Gradient] \label{lem_policy_grad} Consider a two-player zero-sum Markov game, when $x$ is fixed, for $f$ it holds:
 \begin{align*}
     \nabla_{\theta} V^{x, f}(s_0) &= \frac{1}{1-\gamma} \mathbb{E}_{s\sim d_{s_0}^{x, f}} \mathbb{E}_{a \sim x(\cdot|s)} \mathbb{E}_{b \sim f(\cdot|s)} \nabla_{\theta}\log{f(b|s)} Q^{x,f}(s,a,b)\\
     &= \frac{1}{1-\gamma} \mathbb{E}_{s\sim d_{s_0}^{x, f}} \mathbb{E}_{a \sim x(\cdot|s)} \mathbb{E}_{b \sim f(\cdot|s)} \nabla_{\theta}\log{f(b|s)} A^{x,f}(s,a,b)
 \end{align*}
 \end{lemma}
 
 \begin{proof}
 The proof is straightforward.
 \begin{align*}
   & \quad \nabla_{\theta} V^{x, f}(s_0) \\
     &= \nabla_{\theta} \left[ \sum_{a_0} x(a_0|s_0) \sum_{b_0} f(b_0|s_0) Q^{x,f}(s_0, a_0, b_0) \right]\\
     &= \sum_{b_0} \nabla_{\theta} f(b_0|s_0) \cdot \sum_{a_0} x(a_0|s_0) Q^{x,f}(s_0, a_0, b_0) + \mathbb{E}_{a_0} \mathbb{E}_{b_0} \nabla_{\theta} Q^{x, f}(s_0, a_0, b_0)\\
     &= \mathbb{E}_{a_0}\mathbb{E}_{b_0} \left[ \nabla_{\theta} \log{f(b_0|s_0)} Q^{x, f}(s_0, a_0, b_0) \right]
     + \gamma \mathbb{E}_{a_0}\mathbb{E}_{b_0}\mathbb{E}_{s_1} \nabla_{\theta} V^{x, f}(s_1)\\
     &= \mathbb{E}_{x, f} 
     \left[\sum_{t=0}^\infty \gamma^t \nabla_{\theta} \log{f(b_t|s_t)} Q^{x, f}(s_t, a_t, b_t) \right]\\
     &= \frac{1}{1-\gamma} \mathbb{E}_{s\sim d_{s_0}^{x, f}} \mathbb{E}_{a \sim x(\cdot|s)} \mathbb{E}_{b \sim f(\cdot|s)} \nabla_{\theta}\log{f(b|s)} Q^{x,f}(s,a,b)
 \end{align*}
 Notice that, when replacing $Q^{x, f}(s,a,b)$ in the final line with $A^{x, f}(s,a,b)$,
 \begin{align*}
     & \quad \mathbb{E}_{a \sim x(\cdot|s)} \mathbb{E}_{b \sim f(\cdot|s)} \nabla_{\theta}\log{f(b|s)} A^{x,f}(s,a,b)\\
     &= \mathbb{E}_{a \sim x(\cdot|s)} \mathbb{E}_{b \sim f(\cdot|s)} \nabla_{\theta} \log{f(b|s)} (Q^{x,f}(s,a,b) - V^{x, f}(s) )
 \end{align*}
 Note that $\E_{b\sim f(\cdot|s)} \nabla_{\theta} \log{f(b|s)} = 0$, then $V^{x, f}(s)$ term's influence is zero. Proof is completed. 
 \end{proof}

The distribution mismatch coefficient, which is often used for single-agent policy-based optimzation, is a weaker condition compared to concentrability coefficients, see \citep{scherrer2014approximate} for more discussion.
\begin{lemma}[distribution mismatch coefficient and concentrability coefficients] \label{lem_mismatch_coeff}
For any fix policy $x$ and its best response $f^*$, for infinite horizon, it holds
\begin{equation*}
    \left\| \frac{d_\sigma^{x, f^*}}{\sigma} \right\|_\infty \le \frac{1}{1-\gamma} \gC_{\sigma, \sigma}^\prime
\end{equation*}
\end{lemma}

\begin{proof}
 The proof is straightforward, 
 \begin{align*}
     \left\| \frac{d_\sigma^{x, f^*}}{\sigma} \right\|_\infty &= (1-\gamma) \left\| \sum_{m \ge 0} \frac{\gamma^m \sigma (\gP^{x, f^*})^m }{\sigma} \right\|_\infty\\
     &\le (1-\gamma) \sum_{m \ge 0} \gamma^m \cdot \left\| \frac{\gamma^m \sigma (\gP^{x, f^*})^m }{\sigma} \right\|_\infty\\
     &\le (1-\gamma) \sum_{m \ge 1} m \gamma^{m-1} c_{\sigma, \sigma}(m-1)\\
     &\le \frac{\gC_{\sigma, \sigma}^\prime}{1-\gamma}
 \end{align*}
 Proof is completed.
\end{proof}
    
    \section{Proof for Section~\ref{sec:population_algo}}
    \label{sec:pf_popu}
    \textbf{[Proof sketch]}
Recall Approximate Value/Polity Iteration for zero-sum games~\citep{perolat2015approximate}, in each step $k = 1, 2, 3 \cdots$, 
\begin{align}
    & \gT V_{k-1} \le \gT_{x^k} V_{k-1} + \epsilon_k^\prime\\
    & V_k = (\gT_{x^k})^m V_{k-1} + \epsilon_k
\end{align}

where $m$ denotes the number of performing Bellman operators w.r.t. a fixed value-function $V_{k-1}$, and $\epsilon_k$ is tolerance term. Specifically, $m=1$ is Value Iteration. While $m \to \infty$, $(\gT_{x^k})^m V_{k-1} \to V^{x^k} = \inf_{f} V^{x^k, f}$ due to Bellman operator's contraction property (see Lemma~\ref{contraction_bellman}).

We discuss details to characterize error brought by two steps in each iteration, namely: $\epsilon_k^\prime$ and $\epsilon_k$ respectively.

\paragraph{Greedy Step}
% In this step, we aim to give upper bounds of following terms,
% \begin{gather}
%     \sum_{t=1}^{T^\prime} \Big{(} \phi(f_t, x_t) - \phi(f^*, x_t) \Big{)}, \sum_{t=1}^{T^\prime} \Big{(} -\phi(f_t, x_t) - [-\phi(f_t, x^*)] \Big{)}
% \end{gather}

% \paragraph{Analysis}
Suppose two sequences of $\{f_t\}$ and $\{ x_t\}$ is given, let $\Bar{f_{T^\prime}} = \frac{1}{T^\prime} \sum_{t=1}^{T^\prime} f_t, \Bar{x_T^\prime} = \frac{1}{T^\prime} \sum_{t=1}^{T^\prime} x_t$, then
\begin{equation}
    \inf_f \frac{1}{T^\prime} \sum_{t=1}^{T^\prime} \phi(f, x_t) \le \inf_f \phi(f, \Bar{x_{T^\prime}}) \le \sup_x \inf_f \phi(f, x) \le \sup_x \phi(\Bar{f_{T^\prime}}, x) \le \sup_x \frac{1}{T^\prime} \sum_{t=1}^{T^\prime} \phi(f_t, x),
\end{equation}
where $\phi(f, x) = x^\top A f$ in this paper. Specifically, assume that two players produce sequences by using a regret minimization algorithm respectively, of which the bounds are
\begin{gather} \label{Eq_regrets}
    \frac{1}{T^\prime} \sum_{t=1}^{T^\prime} \phi(f_t, x_t) - \inf_f \sum_{t=1}^{T^\prime} \frac{1}{T^\prime} \phi(f, x_t) \le Rate(x_1, x_2,  \cdots x_{T^\prime})\\
    \frac{1}{T^\prime} \sum_{t=1}^{T^\prime} (-\phi(f_t, x_t)) - \inf_x \frac{1}{T^\prime} \sum_{t=1}^{T^\prime} (- \phi(f_t, x)) \le Rate(f_1, f_2,  \cdots f_{T^\prime})
  \end{gather}
Thus, an upper bound of \textbf{Greedy Step} could be derived
\begin{align*}
    \sup_x \inf_f \phi(f, x) - \inf_f \phi(f, \Bar{x_{T^\prime}} ) \le Rate(x_1, x_2,  \cdots x_{T^\prime}) + Rate(f_1, f_2,  \cdots f_{T^\prime})
\end{align*}
Only in this section we denote $(f^*, x^*)$ as the policy pair at NE for simplicity, i.e., $\sup_x \phi(f^*,x ) = \inf_f \sup_x \phi(f, x) ) = \inf_f \phi(f, x^*)$.
\begin{lemma}[Greedy step suboptimality.]\label{lem_popu_greedy_opt} In Algorithm~\ref{alg:popu_MD}, when both players adopt the following adaptive step sizes for each state $s \in \gS$,
\begin{gather*}
    \eta_t^s = \min \Bigg{\{} \frac{\log(|\gA|{T^\prime}^2)}{ \sqrt{\sum_{i=1}^{t-1} \|A_s^\top x_i(\cdot|s) - A_s^\top x_{i-1}(\cdot|s)\|_\star^2} + \sqrt{\sum_{i=1}^{t-2} \|A_s^\top x_i(\cdot|s) - A_s^\top x_{i-1}(\cdot|s) \|_\star^2} }, \frac{1}{1 + \frac{10}{(1-\gamma)^2}} \Bigg{\}},\\
    {\eta_t^s}^\prime = \min \Bigg{\{} \frac{\log(|\gA|{T^\prime}^2)}{ \sqrt{\sum_{i=1}^{t-1} \|A_s f_i(\cdot|s) - A_s f_{i-1}(\cdot|s)\|_\star^2} + \sqrt{\sum_{i=1}^{t-2} \|A_s f_i(\cdot|s) - A_s f_{i-1}(\cdot|s)\|_\star^2} }, \frac{1}{1 + \frac{10}{(1-\gamma)^2}} \Bigg{\}},
\end{gather*}
then pair $(\Bar{x_{T^\prime}}, \Bar{f_{T^\prime}})$ is an $\Tilde{O} \left( \frac{\log{|\gA| + \log}{T^\prime}}{(1-\gamma)^2 T^\prime} \right)-$ approximate minimax equilibrium.

\end{lemma}
\begin{proof} Proof is modified from~\citep{rakhlin2013optimization}. Regret minimization procedure in Eq.~\ref{Eq_regrets} is calculated as:
\begin{align*}
    & \quad \sum_{t=1}^{T^\prime} \sup_x \phi(f_t, x) - \sup_x \phi(f^*, x)\\
    &\le \sum_{t=1}^{T^\prime}\left \langle f_t - f^*, \nabla_f \left(\sup_x \phi(f_t, x)\right) \right \rangle \\\
    &\le \left( \frac{1}{\eta_1} \right)R_{max}^2 + \sum_{t=1}^{T^\prime} \| A^\top x_t - A^\top x_{t-1} \|_\star \| g_t -f_t\|\\
    & - \frac{1}{2} \sum_{t=1}^{T^\prime} \frac{1}{\eta_t} (\|g_t^\prime - f_t\|^2 + \|g_{t-1}^\prime - f_t \|) + 1,
\end{align*}
where $R_{max}^2$ is upper bound of KL divergence between $f^*$ and any $g^\prime$, so $R_{max}^2 = \log(|\gA|{T^\prime}^2).$ With some calculations, the sum of two regrets (Eq.~\ref{Eq_regrets} and its counterpart) is upper bounded by
\begin{align}
    6 + \left( 4 + \frac{40}{(1-\gamma)^2}\right) \log(|\gA| {T^\prime}^2) + \frac{1}{T^\prime} \frac{40}{(1-\gamma)^2}
\end{align}
Thus regret is upper bounded by $\gO \left( \frac{\log{|\gA|} + \log{{T^\prime}}}{(1-\gamma)^2 T^\prime} \right)$.

\end{proof}

\paragraph{Iteration Step}
While Approximate Value-based algorithms generally focus on the relation between $\epsilon_k$ and accumulative error of $\epsilon_{k, i}, i = 1,2,3 \cdots m$. We could take another view of this iteration: at $k^{th}$ iteration, let $V^{x^k}$ be our goal to achieve, $V_k$ is what we finally get with optimization techniques, and $\epsilon_k$ now turns out to be a suboptimality gap.
\begin{align*}
    \sum_s |\epsilon_k(s)| \sigma(s) =  V^{x, f^{T}}(\sigma) - \inf_{f} V^{x^k, f} (\sigma)
\end{align*}
Describe this with general policy-based languages: when we are given $x^k$, we desire to find $\inf_{f} V^{x^k, f}$. Note that it holds similarity to the general single-agent MDP, where the agent seeks to find $\max_{\pi} V^{\pi}$. Thus we could apply NPG for player two at iteration step, next we show NPG update in Algorithm~\ref{alg:popu_NPG} takes exponential form on the weighted advantage function.

\begin{lemma} Fix $x$, when $f$ is softmax parameterized, it holds
\begin{align*}
    f^{t+1} \propto f^{t} \cdot \exp^{- \frac{\eta}{1-\gamma} \sum_a x(a|s) A^{x, f}(s,a,b)}
\end{align*}
\end{lemma}
\begin{proof}
Notice that Fisher matrix calculation for this case is often obtained by minimizing
\begin{equation*}
    L(w) = \E_{s \sim d_\sigma^{x, f}} \E_{b\sim f} \left(w^\top \nabla_{\theta} \log{f(b|s) - \sum_a x(a|s) A^{x, f}(s,a,b)}\right)^2,
\end{equation*}
which is because:\\
At minimizer $w^*$, $\frac{d L(w^*)}{dw} = 0$ implies that
\begin{align*}
    \E_{s \sim d_\sigma^{x, f}} \E_{b\sim f} \left( (w^*)^\top  \nabla_{\theta} \log{f(b|s) - \sum_a x(a|s) A^{x, f}(s,a,b)} \right) \nabla_{\theta} \log{f(b|s)} =0,
\end{align*}
rearrange this,
\begin{equation*}
    w^* = (1-\gamma) F_{\sigma}(\theta)^\dag \nabla_{\theta} V(\sigma)
\end{equation*}
For softmax parameterization, note: 
\begin{equation*}
    w^* = \sum_a x(a|s) A^{x, f}(s,a,b) + v(s) \Leftrightarrow L(w^*)=0
\end{equation*}

Then NPG updates take the following form,
\begin{gather*}
\theta^{t+1} = \theta^t - \frac{\eta}{1-\gamma} \sum_a x(a|s) A^{x, f}(s,a,b) - \frac{\eta}{1-\gamma} v\\
    f^{t+1} \propto f^{t} \cdot \exp^{- \frac{\eta}{1-\gamma} \sum_a x(a|s) A^{x, f}(s,a,b)}
\end{gather*}
Proof is completed.
\end{proof}

\begin{lemma} With above update rule, after T iterations
\begin{align*}
   V^{x, f^{T}}(\sigma) - V^{x, f^*}(\sigma) \le  \frac{\log{|\gA|}}{\eta T} + \frac{1}{(1-\gamma)^2 T},
\end{align*}
where $f^*$ is player two's best response w.r.t. $x$, which satisfies $V^{x, f^*} (\sigma) = \inf_{f} V^{x, f}(\sigma)$
\end{lemma}

\begin{proof}
Proof sketch is analogous to \citet[Theorem 5.3]{agarwal2020optimality}, only need to replace $A^\pi(s,a)$ with an average term $\sum_a x(a|s) A^{x, f}(s,a,b)$.
\end{proof}

~\\
With these policy optimization results, the proof of Theorem~\ref{thm:popu_NPG} is concise.

\textbf{[Proof for Theorem~\ref{thm:popu_NPG}]} \label{popu_thm_pf}
\begin{proof} After T steps of NPG descent, it can be guaranteed that
\begin{align*}
    \epsilon_k(s) &= V^{x^k, f^T}(s) - \inf_{f} V^{x^k, f} > 0 \\
    \sum_s |\epsilon_k(s)| \sigma(s) &=  V^{x, f^{T}}(\sigma) - \inf_{f} V^{x^k, f} (\sigma) \le \frac{2}{(1-\gamma)^2 T},
\end{align*}
 where we have set $\eta \ge (1-\gamma)^2 \log|\gA|$. Substitute this NPG suboptimality and Greedy step suboptimality (Lemma~\ref{lem_popu_greedy_opt}) into Lemma~\ref{original_form}, proof is completed.
 \end{proof}
 
 \vspace{4em}

 \subsection{Proof for Entropy regularization} \label{sec:pf_popu_entropy}
  First, note that the entropy term w.r.t. min player $f$
 \begin{gather*}
 \gH(\sigma, f) = \frac{1}{1-\gamma} \E_{s \sim d_\sigma^{x, f}} \E_{b \sim f} \log{\frac{1}{f(b|s)}}
 \end{gather*}
 lies in $\big{[} 0, \log{|\gA|} \big{]}$. Recall that $V_\tau^*(\sigma) = \min_{f} V_\tau^{x, f}(\sigma) = V^{x, f_\tau^*}(\sigma)- \tau \gH(\sigma, f_\tau^*)$ and the following sandwich bound holds:
 \begin{align*}
 V^{x, f_\tau^*}(\sigma) &\ge V^{x, f^*(x)}(\sigma) \ge V_\tau^{x, f^*(x)}(\sigma) \ge V_\tau^*(\sigma) \ge V^{x, f_\tau^*}(\sigma) - \frac{\tau}{1-\gamma} \log{}|\gA|,
 \end{align*}

 We aim to bound $V^{x, f^T}(\sigma) -V^{x, f^*(x)}$ through optimizing $V_\tau$, denote $V_\tau^* = V^{x, f^*(x)}$ for short, observe
 \begin{align*}
 V^{x, f^T}(\sigma) - V^{x, f^*(x)}(\sigma) &= V^{x, f^T}(\sigma) - V_\tau^{x, f^T}(\sigma) + V_\tau^{x, f^T}(\sigma) - V_\tau^*(\sigma) \\
 &\le \frac{\tau}{1-\gamma} \log{|\gA|} + V_\tau^{x, f^T}(\sigma) - V_\tau^*(\sigma) + 0.
 \end{align*}
 
 Besides the notations in Section~\ref{sec:population_algo}, introduce the regularized advantage function for min player $f$
 \begin{align*}
     A_\tau^{x, f}(s, a, b) = Q_\tau^{x, f}(s, a, b) + \tau \log{f(b|s)} - V_\tau^{x, f}(s)
 \end{align*}
 Regularized reward is
 \begin{align*}
     r_\tau(s, a, b) = r(s, a, b) + \tau \log{f(b|s)}
 \end{align*}
 \begin{lemma}[Regularized policy gradients]
 \begin{align*}
     \nabla_{\theta} V_\tau^{x, f}(s_0) = \frac{1}{1-\gamma} \E_{s \sim d_{s_0}^{x, f}} \E_{a \sim x} \E_{b \sim f} \nabla_{\theta} \log{f(b|s)} A_\tau^{x, f}(s, a, b)
 \end{align*}
 \begin{proof}
 Note that soft Q function is $Q_\tau^{x, f}(a, a, b) = r(s, a, b) + \gamma \E_{s^\prime} V_\tau^{x, f}(s^\prime)$
 \begin{align*}
     \nabla_{\theta} V_\tau^{x, f}(s_0) &= \nabla_{\theta} \left[ \sum_{b_0} f(b_0|s_0) \sum_{a_0} x(a_0|s_0) \bigg{(} r(s_0, a_0, b_0) + \tau \log{f(b_0|s_0)} + \gamma \E_{s_1} V_\tau^{x, f}(s_1) \bigg{)} \right]\\
     &= \nabla_{\theta} \left[ \sum_{b_0} f(b_0|s_0) \E_{a_0 \sim x} \bigg{(} Q_\tau^{x, f}(s_0, a_0, b_0) +\tau \log{f(b_0|s_0)} \bigg{)} \right]\\
     &= \sum_{b_0} f(b_0|s_0) \nabla_{\theta} \log{f(b_0|s_0)} \E_{a_0 \sim x} \bigg{(} Q_\tau^{x, f}(s_0, a_0, b_0) +\tau \log{f(b_0|s_0)} \bigg{)} \\
     & \qquad \qquad + \E_{b_0 \sim f} \E_{a_0 \sim x} \nabla_{\theta} \bigg{(} r(s_0, a_0, b_0) + \gamma \E_{s_1} V_\tau^{x, f}(s_1) +\tau \log{f(b_0|s_0)} \bigg{)}\\
     &= \E \left[ \sum_{t=0}^\infty \gamma^t \nabla_{\theta} \log{f(b_t|s_t)} \bigg{(} Q_\tau^{x, f}(s_t, a_t, b_t) + \tau \log{f(b_t|s_t)} \bigg{)} \right]\\
     &= \frac{1}{1-\gamma} \E_{s \sim d_{s_0}^{x, f}} \E_{a \sim x} \E_{b \sim f} \nabla_{\theta} \log{f(b|s)} \bigg{(} Q_\tau^{x, f}(s, a, b) + \tau \log{f(b|s)}\bigg{)}\\
     &= \frac{1}{1-\gamma} \E_{s \sim d_{s_0}^{x, f}} \E_{a \sim x} \E_{b \sim f} \nabla_{\theta} \log{f(b|s)} A_\tau^{x, f}(s, a, b)
 \end{align*}
 \end{proof}
 \end{lemma}
 Adopting softmax parameterization, gradient is written as
 \begin{align*}
     \frac{\partial V_\tau^{x, f}(s_0)}{\partial \theta(s, b)} = \frac{1}{1-\gamma} d_{s_0}^{x, f}(s) f(b|s) \E_{a \sim x} A_\tau^{x, f}(s, a, b)
 \end{align*}
 \begin{lemma}[Regularized update rule]
 \begin{align}
     f^{t+1}(b|s) \propto \left( f^t(b|s)\right)^{1- \frac{\eta \tau}{1-\gamma}} \exp{\bigg{(} -\frac{\eta}{1-\gamma} \sum_a x(a|s) Q_\tau^{x, f}(s, a, b) \bigg{)}}
 \end{align}
 \end{lemma}
 \begin{proof}
 Denote update direction as $w^* = (F_\sigma^{\theta})^\dag \nabla_{\theta} V_\tau^{x, f}(\sigma)$, which means $w^*$ is minimizer of square loss 
 \begin{align*}
 & \qquad \left\|F_\sigma^{\theta} w - \nabla_{\theta} V_\tau^{x, f}(\sigma) \right\|^2\\
 &= \sum_{s, b} \left( d_\sigma^{x, f}(s) f(b|s) (w_{s, b} - c(s)) - \frac{1}{1-\gamma} d_\sigma^{x, f} f(b|s) \E_{a \sim x} A_\tau^{x, f}(s, a, b) \right)^2,
 \end{align*}
 thus $w_{s, b} = c(s) + \frac{1}{1-\gamma} \E_{a \sim x} A^{x, f}(s, a, b)$, and 
 \begin{align*}
     f^{t+1}(b|s) &\propto f^t(b|s) \exp{ \left( -\frac{\eta}{1-\gamma} \E_{a \sim x} A_\tau^{(t)}(s, a, b) \right)}\\
     &= f^t(b|s) \exp{ \left( -\frac{\eta}{1-\gamma} \E_{a \sim x} (Q_\tau^{(t)}(s, a, b)+\tau \log{f(b|s)} - V_\tau^{(t)}(s)) \right)}\\
     &\propto \left(f^t(b|s)\right)^{1- \frac{\eta \tau}{1-\gamma}} \exp{\left( -\frac{\eta}{1-\gamma} \sum_a x(a|s) Q_\tau^{(t)}(s, a, b) \right)}
 \end{align*}
 Proof is completed.
 \end{proof}
 
 \begin{lemma}[Performance improvement lemma]
 \begin{align*}
     V_\tau^t(s_0) = V_\tau^{t+1}(s_0) + \E_{s \sim d_{s_0}^{t+1}} \left[ (\frac{1}{\eta} - \frac{\tau}{1-\gamma}) KL \left( f^{t+1}(\cdot|s) \| f^t(\cdot|s) \right) + \frac{1}{\eta} KL \left(f^t(\cdot|s) \| f^{t+1}(\cdot|s) \right) \right]
 \end{align*}
 \begin{proof}
 Regularized update rule can be transformed to 
 \begin{align*}
     \frac{1-\gamma}{\eta} \left( \log{f^{t+1}(b|s)} - \log{f^t(b|s)} \right) + \frac{1-\gamma}{\eta} \log{Z^t(s)} = - \tau \log{f^t(b|s)} - \E_{a\sim x} Q_\tau^t(s, a, b)
 \end{align*}
 Then 
 \begin{align*}
     V_\tau^t(s_0) &= \E_{a\sim x} \E_{b \sim f^t} \left[ \tau \log{f^t(b_0|s_0)} + Q_\tau^t(s_0, a_0, b_0) \right]\\
     &= -\frac{1-\gamma}{\eta} \log{Z^t(s_0)} + \frac{1-\gamma}{\eta} KL \left(f^t(\cdot\| s_0), f^{t+1}(\cdot\| s_0) \right)\\
     &= \E_{b_0 \sim f^{t+1}} \left[ \tau \log{f^t(b|s)} + \E_{a \sim x} Q_\tau^t(s, a, b) + \frac{1-\gamma}{\eta}\left( \log{f^{t+1}(b|s)} - \log{f^t(b|s)} \right) \right]\\
     & \qquad + \frac{1-\gamma}{\eta} KL \left( f^t(\cdot|s_0) \| f^{t+1}(\cdot|s_0) \right)\\
     &= \E_{b_0 \sim f^{t+1}} \left[ \tau \log{f^{t+1}(b_0|s_0)} + \E_{a \sim x} Q_\tau^t(s_0, a_0, b_0) \right]\\
     & \qquad+ \left( \frac{1-\gamma}{\eta} -\tau \right) KL \left( f^{t+1}(\cdot|s_0) \| f^{t}(\cdot|s_0) \right) + \frac{1-\gamma}{\eta} KL \left( f^t(\cdot|s_0) \| f^{t+1}(\cdot|s_0) \right)
 \end{align*}
 Note that: $Q_\tau^t(s_0, a_0, b_0) = r(s_0, a_0, b_0) + \gamma \E_{s_1} V_\tau^t(s_1)$, apply this recurrently then the proof is completed.
 \end{proof}
 
 \end{lemma}
 
 For regularized Markov games, the suboptimality gap is shown to be
 \begin{align*}
     &\qquad V_\tau^{x, f_\tau^*}(\sigma) - V_\tau^{x, f^t}(\sigma)\\
     &= \E \left[ \sum_{i=0}^\infty \gamma^i \left(r(s_i, a_i, b_i) + \tau \log{f_\tau^*(b_i|s_i)} \right) \right] - V_\tau^{x, f^t}(\sigma)\\
     &= \E \left[ \sum_{i=0}^\infty \gamma^i \left( r(s_i, a_i, b_i) + \tau \log{f_\tau^*(b_i|s_i)} + \gamma V_\tau^{x, f^t}(s_{i+1}) - V_\tau^{x, f^t}(s_i) \right) \right]\\
     &= \frac{1}{1-\gamma} \E_{s \sim d_\sigma^{x, f_\tau^*}} \left[ \sum_b f_\tau^*(b|s) \left( \E_{a \sim x} Q_\tau^{(t)}(s, a, b) + \tau \log{f_\tau^*(b|s)} \right) - V_\tau^{(t)}(s) \right].
 \end{align*}
 Take reverse, 
 \begin{align*}
     &\qquad V_\tau^{x, f^t}(\sigma) - V_\tau^{x, f_\tau^*}(\sigma)\\
     &= \frac{1}{1-\gamma} \E_{s \sim d_\sigma^{x, f_\tau^*}} \left[ V_\tau^{(t)}(s) + \sum_b f_\tau^*(b|s) \left( -\E_{a \sim x} Q_\tau^{(t)} - \tau \log{f_\tau^*(b|s)}\right) \right],
 \end{align*}
 where
 \begin{align*}
     & \qquad \sum_b f_\tau^*(b|s) \left( -\E_{a \sim x} Q_\tau^{(t)}(s, a, b) - \tau \log{f_\tau^*(b|s)}\right)\\
         &= \tau \sum_b f_\tau^*(b|s) \log{\left( \frac{e^{\E_{a \sim x}  - \nicefrac{Q_\tau^{(t)}(s, a, b)}{\tau}}}{f_\tau^*(b|s)} \right)}\\
         &\le \tau \log{\sum_b \exp{ \left( - \E_{ a\sim x} \frac{Q_\tau^{(t)}(s, a, b)}{\tau} \right)}}
 \end{align*}
 
\paragraph{Contraction property} Following contraction argument raised by \citet{cen2020fast}, suppose $\eta = \frac{1-\gamma}{\tau}$
 \begin{align*}
     & \quad V_\tau^{x, f^{t+1}}(\sigma) - V_\tau^{x, f_\tau^*}(\sigma)\\
     &= V_\tau^{x, f^{t+1}}(\sigma) - V_\tau^{x, f^{t}}(\sigma) + V_\tau^{x, f^{t}}(\sigma) - V_\tau^{x, f_\tau^*}(\sigma)\\
     &= \E_{s \sim d_\rho^{t+1}} - \frac{1}{\eta} KL \left( f^t(\cdot|s) 
     \| f^{t+1}(\cdot|s) \right) + V_\tau^{x, f^{t}}(\sigma) - V_\tau^{x, f_\tau^*}(\sigma)\\
     &\le \left( V_\tau^{x, f^{t}}(\sigma) - V_\tau^{x, f_\tau^*}(\sigma) \right) \cdot \left( 1 - \left\| \frac{d_\rho^{x, f_\tau^*}}{d_\rho^{t+1}} \right\|_\infty^{-1} \right)\\
     &\le \left( V_\tau^{x, f^{t}}(\sigma) - V_\tau^{x, f_\tau^*}(\sigma) \right) \cdot \left[ 1 - (1-\gamma)\left\| \frac{d_\rho^{x, f_\tau^*}}{\rho} \right\|_\infty^{-1} \right]
 \end{align*}
 
  Denote \emph{stationary distribution} as: $\mu_\tau^* = d_{\mu_\tau^*}^{x, f_\tau^*}$ and
  \begin{align*}
      & \quad V_\tau^{x, f^{t}}(\sigma) - V_\tau^{x, f_\tau^*}(\sigma)\\
      &\le \left\| \frac{\sigma}{\mu_\tau^*} \right\|_\infty \cdot \left( V_\tau^{x, f^{t}}(\mu_\tau^*) - V_\tau^{x, f_\tau^*}(\mu_\tau^*) \right)\\
      &\le \left\| \frac{\sigma}{\mu_\tau^*} \right\|_\infty \cdot \gamma^t \left( V_\tau^{x, f^{0}}(\mu_\tau^*) - V_\tau^{x, f_\tau^*}(\mu_\tau^*) \right)
  \end{align*}
  
  Combining these results, we are ready to show Theorem~\ref{thm:popu_NPG_entropy}.\\
  \textbf{[Proof for Theorem~\ref{thm:popu_NPG_entropy}]}
  \begin{proof}
   \begin{align*}
    &\quad V^{x, f^T}(\sigma) - V^{x, f^*(x)}(\sigma)\\
    &= V^{x, f^T}(\sigma) - V_\tau^{x, f^T}(\sigma) + V_\tau^{x, f^T}(\sigma) - V_\tau^*(\sigma) + V_\tau^*(\sigma) -V^{x, f^*(x)}(\sigma)\\
    &\le \frac{\tau \log{|\gA|}}{1-\gamma} + \left\| \frac{\sigma}{\mu_\tau^*} \right\|_\infty \cdot \gamma^T \left( V_\tau^{x, f^{0}}(\mu_\tau^*) - V_\tau^{x, f_\tau^*}(\mu_\tau^*) \right),
\end{align*}
note that $V_\tau^{x, f^{0}}(\mu_\tau^*) - V_\tau^{x, f_\tau^*}(\mu_\tau^*) \le 1 + \tau \log{|\gA|}.$

Proof is completed via substitution into Lemma~\ref{original_form}.
\end{proof}

    \section{Proof for Section~\ref{sec:online_algo}}
    \label{sec:pf_online}
    For online setting, we consider Approximate Generalized Policy Iteration (Eq.\ref{GPI_greedy}, Eq.\ref{GPI_iter}) in expectation
\begin{gather}
        \E [\gT V_{k-1}] \le \E[\gT_{x^k} V_{k-1}] + \E[\epsilon_k^\prime]\\
            \E[V_k] = \E\left[\left( \gT_{x^k} \right)^m V_{k-1}\right] + \E[\epsilon_k]
    \end{gather}
Here, we try to bound the summation of tolerances over state space $\gS$. In function approximation, $\gS$ could be very large or even infinite. Therefore, we use the optimization measure $\sigma$ we take to train our policy for generalization across states, namely:
\begin{align*}
    \E[\epsilon_k^\prime] &= \E[\sum_s \sigma(s) \epsilon_k^\prime(s)]\\
    \E[\epsilon_k] &= \E[\sum_s \sigma(s) \epsilon_k(s)]
\end{align*}
Randomness is brought by oracle sampling and stochastic optimization.

Based on the iterative scheme, Lemma~\ref{original_form} could be adapted to expectation: after $k$ iterations,
\begin{align*}
        \E \left[V^*(\rho) - \inf_{f} V^{x^k, f}(\rho) \right] &\le \frac{2(\gamma-\gamma^k) \gC_{\rho, \sigma}^{1, k, 0}}{(1-\gamma)^2} 
     \epsilon + \frac{(1-\gamma^k)\gC_{\rho, \sigma}^{0, k, 0}}{(1-\gamma)^2}  \epsilon^\prime
        + \frac{ 2 \gamma^k \gC_{\rho, \sigma}^{k, k+1, 0}}{1-\gamma}  ,
    \end{align*}
    where
    \begin{gather*}
        \epsilon = \sup_{1 \le j \le k-1} \E[\epsilon_j],\\
        \epsilon^\prime = \sup_{1 \le j \le k } \E[\epsilon_j^\prime].
    \end{gather*}
We are able to derive suboptimality gap for the online setting.

\textbf{[Proof sketch]} Similar to Section~\ref{sec:pf_popu}, discuss errors brought by two phases respectively.
\paragraph{Greedy Step}
~\\
\emph{Problem Restatement:} Consider a two-player zero-sum matrix game, formally 
\begin{gather*}
    \min_{f(\cdot \mid s) \in \Delta(|\gA|)} \max_{x(\cdot \mid s) \in \Delta(|\gA|)} f^\top A_s x\\
    A_s(a, b) = r(s, a, b) + \sum_{s^\prime} \gP(s^\prime \mid s, a, b) V_{k-1}(s^\prime) 
\end{gather*}
The goal is to output policy $\Bar{x_{T^\prime}}$ for max player and an upper bound of \textbf{Greedy error} $\E[\epsilon^\prime] = \E[\sum_s \sigma(s) \epsilon_k^\prime (s)]$:
\begin{align*}
    & \quad \E \left[ \sum_s \sup_x \inf_f \phi_s(f(\cdot|s), x(\cdot|s)) - \inf_f \phi_s(f(\cdot|s), \Bar{x_{T^\prime}}(\cdot | s)) \right]\\
    & \le \frac{1}{T^\prime} \E \sum_s \Bigg{\{} \sum_{t=1}^{T^\prime} \phi_s(f_t, x_t) - \inf_f \sum_{t=1}^{T^\prime} \phi_s(f, x_t) + \sum_{t=1}^{T^\prime} (-\phi_s(f_t, x_t)) - \inf_x \sum_{t=1}^{T^\prime} (-\phi_s(f_t, x)) \Bigg{\}}
\end{align*}
We only analyze max player ($x^t$) and min player ($f^t$) is very similar.

First, we show our Algorithm~\ref{alg:online_MD} is using unbiased gradient estimates. Observe:
\begin{align*}
    & \quad \E [g_n]\\
    &= \E_{s\sim \sigma} \E_{a \sim x^t(\cdot|s)} \E_{b \sim f^t(\cdot|s)} \E_{s^\prime \sim  \gP(\cdot|s,a,b)} \E_{a^\prime \sim x^t(\cdot|s)} [r(s,a,b) +\gamma V_{k-1}(s^\prime)]\\
    &\qquad \qquad \cdot (\nabla_\xi \log{x^t(a|s)} - \nabla_\xi \log{x_t(a^\prime|s)})\\
    &= \E_{s\sim \sigma} \E_{a \sim x^t(\cdot|s)} (A_s f^t)_a \nabla_\xi \log{x^t(a|s)} - \E_{s\sim \sigma} \E_{a^\prime \sim x^t(\cdot|s)} \phi_s(f_t, x_t) \nabla_\xi \log{x^t(a^\prime|s)}\\
    &= \E_{s\sim \sigma} \E_{a \sim x^t(\cdot|s)} [(A_s f_t)_a - \phi_s(f_t, x_t)] \nabla_\xi \log{x^t(a|s)}.
\end{align*}
Recall notations raised in Eq.~\ref{eq_MD}, where $x^*$ is \emph{best response} of average policy $\frac{1}{T^\prime} \sum_{t=1}^{T^\prime} f^t$,
\begin{align*}
    & \qquad \E_{s \sim \sigma} KL(x^*(\cdot|s) \| x^t(\cdot|s)) - KL(x^*(\cdot|s) \| x^{t+1}(\cdot|s)) \\
    &= \E_{s \sim \sigma} \E_{a \sim x^*} \log \frac{x^{t+1}(a|s)}{ x^t (a|s)}\\
    &\ge \E_{s \sim \sigma} \E_{a \sim x^*(\cdot|s)} \langle \nabla_\theta \log x^t (a|s), \eta^\prime \hat{w}^t \rangle - \frac{\beta}{2} \| \xi^{t+1} - \xi^t \|^2\\
    &= \eta^\prime \E_{s \sim \sigma}\E_{a \sim x^*} \left[ \nabla_\theta \log{x^t(a|s)}^\top \hat{w}^t - \left( (A_s f_t)_a - \phi_s(f_t, x_t) \right) \right]\\
    & \qquad +\eta^\prime \E_{s \sim \sigma} \E_{a \sim x^*(\cdot|s)} \left[ (A_s f_t)_a - \phi_s(f_t, x_t) \right] - \frac{\beta {\eta^\prime}^2 {W}^2}{2}\\
    &\ge - \eta^\prime \sqrt{\E_{s \sim \sigma}\E_{a \sim x^*(\cdot|s)} \left( \nabla_\theta \log{x^t (a|s)}^\top \hat{w}^t - \left[ (A_s f_t)_a - \phi_s(f_t, x_t) \right] \right)^2}\\
    &\qquad  + \eta^\prime \E_{s \sim \sigma}(\phi_s(f_t, x^*) - \phi_s(f_t, x_t)) - \frac{\beta {\eta^\prime}^2 {W}^2}{2},
\end{align*}
where we define $L(w^t) = \E_{s \sim \sigma}\E_{a \sim x^t} \left( \nabla_\theta \log{x^t (a|s)}^\top w^t - \left[ (A_s f_t)_a - \phi_s(f_t, x_t) \right] \right)^2$ with little abuse of notation.

Rearrange inequality and the upper bound of $\E_{s \sim \sigma} \phi_s(f_t, x^*) - \phi_s(f_t, x_t)$ is smaller than
\begin{equation*}
      \frac{1}{\eta^\prime} \E_{s \sim \sigma} \left[ KL(x^*(\cdot|s) \| x^t(\cdot|s)) - KL(x^*(\cdot|s) \| x^{t+1}(\cdot|s)) \right] + \sqrt{\sup_{t \le T^\prime} \left\|  \frac{1}{x^t} \right\|_\infty} \sqrt{L(\hat{w}^t)} + \frac{\beta}{2}\eta^\prime {W}^2
\end{equation*}

Expectation of $\epsilon_{est}$ is bounded by sample complexity, SGD optimizer has
\begin{align*}
    \epsilon_{est} = \E [L(\hat{w}^t)] - L(w^*) \le \frac{G W}{\sqrt{N^\prime}},
\end{align*}
where $G = 2B(BW + \nicefrac{2}{1-\gamma})$ bounds norm of gradient estimation, learning rate $\alpha^\prime$ is set as $\nicefrac{W}{G \sqrt{N^\prime}}$, see Lemma~\ref{lem_stat_error}. 

Thus
\begin{align*}
    & \quad \E \left[ \sup_x \inf_f \phi_s(f(\cdot|s), x(\cdot|s)) - \inf_f \phi_s(f(\cdot|s), \Bar{x_{T^\prime}}(\cdot | s)) \right] \\
    & \le \frac{1}{\eta^\prime} \frac{1}{T^\prime}\E_{s \sim \sigma} \left[ KL(x^*(\cdot|s) \| x^1(\cdot|s)) + KL(f^*(\cdot|s) \| f^1(\cdot|s)) \right] \\
    & \qquad + \left(\sqrt{\sup_{t \le T^\prime} \left\| \frac{1}{f^t} \right\|_\infty}+\sqrt{\sup_{t \le T^\prime} \left\| \frac{1}{x^t} \right\|_\infty}\right) \cdot \sqrt{\epsilon_{approx}^\prime + \frac{G W}{\sqrt{N^\prime}}} + \beta \eta^\prime {W}^2\\
    &\le \frac{2 \log{|\gA|}}{\eta^\prime T^\prime} + \left(\sqrt{\sup_{t \le T^\prime} \left\| \frac{1}{f^t} \right\|_\infty}+\sqrt{\sup_{t \le T^\prime} \left\| \frac{1}{x^t} \right\|_\infty}\right) \cdot \sqrt{\epsilon_{approx}^\prime + \frac{G W}{\sqrt{N^\prime}}} + \beta \eta^\prime {W}^2
\end{align*}
Let $\eta^\prime = \sqrt{\frac{2\log{|\gA|}}{\beta {W}^2 T^\prime}}$, and finally $\E [\epsilon^\prime] = \E [\sum_s \sigma(s) \epsilon_k^\prime(s)]$ is lower than
\begin{align}
     2\sqrt{\frac{2\log{|\gA| \beta {W}^2}}{T^\prime}} + 2\iota\left(\sqrt{\epsilon_{approx}^\prime} + \frac{\sqrt{G W}}{{N^\prime}^{\frac{1}{4}}}\right)
\end{align}

\paragraph{Iteration Step}
 See NPG regret Lemma~\ref{lem4_NPG_regret} for two-player zero-sum games, where $err_t$ is bounded  when $\nu_0(s, a, b) = \nicefrac{\sigma(s)} {|\gA|^2}$ is an exploration distribution covering all states and actions:
\begin{align*}
    |err_t| &\le \sqrt{  \E_{s \sim d_\sigma^{x, f^*}, a\sim x, b\sim f^*(x)}
   \left[ A^{x, f^t} (s, a, b) - w^{t} \nabla_{\theta} \log{f^{t} (b|s)} \right]^2 } \\
   &\le \sqrt{ \left\| \frac{d_\sigma^{x, f^*}\cdot x \cdot f^*}{\nu^t} \right\|_\infty    \E_{s, a, b \sim \nu^t} \left( A^{x, f^t}(s, a, b) - w^t \nabla_{\theta} \log{f^t(b|s)} \right)^2        }\\
    &\le \sqrt{\frac{|\gA|^2}{1-\gamma} \left\| \frac{d_\sigma^{x, f^*}}{\sigma} \right\|_\infty L(\hat{w}^t, \theta)},
\end{align*}
Notice $\E  \sqrt{\frac{1}{T} \sum_t L(\hat{w}^t, \theta) } \le \sqrt{  \frac{1}{T} \sum_t \E[L(\hat{w}^t, \theta)]} $,
 then proof is completed via upper bounding $ \E[L(\hat{w}^t, \theta)] $.\\
 The final equality contains \emph{distribution mismatch coefficient} $ \left\| \nicefrac{d_\sigma^{x, f^*}}{\sigma} \right\|_\infty$, which often appears in single-agent policy-based optimization. It measures the difficulty of exploration problems faced by algorithms. Furthermore, \emph{concentrability coefficients} are stronger, from which $ \left\| \nicefrac{d_\sigma^{x, f^*}}{\sigma} \right\|_\infty$ could be derived. See Lemma~\ref{lem_mismatch_coeff}.

We first introduce a two-player zero-sum Markov game version regret lemma, single agent version of MDP is useful for online NPG analysis~\citep{agarwal2020optimality}.

\begin{lemma}[NPG regret] \label{lem4_NPG_regret}
Assume for all $s \in \gS $ and $b \in \gA$ that $\log f(b|s)$ is a $\beta$-smooth function, then
\begin{align*}
    \frac{1}{T} \sum_{t=0}^{T-1} V^{x, f^{t}}(\sigma) - V^{x, f^*(x)}(\sigma)  \le \frac{1}{1-\gamma} \left( \frac{\log|\gA|}{\eta T} + \frac{\eta \beta W^2}{2} - \frac{1}{T} \sum_{t=0}^{T-1} err_t \right),
\end{align*}
where $err_t$ is defined as
\begin{align*}
    err_t &= \E_{s \sim d_\sigma^{x, f^*(x)}} \E_{b \sim f^*(x)} \left[ \sum_a x(a|s) A^{x, f^t} (s, a, b) - w^{t} \nabla_{\theta} \log{f^{t} (b|s)} \right]\\
    &= \E_{s, a, b}^* \left[ A^{x, f^t} (s, a, b) - w^{t} \nabla_{\theta} \log{f^{t} (b|s)} \right]
\end{align*}
where we denote $\E_{s,a,b}^* \coloneqq \E_{s \sim d_\sigma^{x, f^*}} \E_{a \sim x} \E_{b \sim f^*} $ for simplicity.

\end{lemma}

\begin{proof}
When making no abuse of notation, we denote $f^*$ as the \emph{best response} of fixed $x$ for simplicity from now on, i.e., $V^{x,f^*} = \inf_f V^{x,f}$

\begin{align*}
    & \quad \E_{s,a,b}^* \left(  KL(f^* \| f^t) - KL(f^* \| f^{t+1}) \right)  \\
    &= \E_{s,a,b}^* \log{\frac{f^{t+1}(b|s) }{f^{t}(b|s)}} \\
    &\ge \E_{s,a,b}^* \left[-\eta \nabla_{\theta} \log{f^t(b|s) w^t} - \frac{\beta \eta^2}{2} W^2 \right]\\
    &= -\eta \E_{s,a,b}^* A^{x, f^t}(s,a,b) + \eta \E_{s,a,b}^* \left(A^{x, f^*}(s, a, b) - \nabla_{\theta} \log{f^t(b|s) }\right) - \frac{\beta \eta^2 W^2}{2}\\
    &= -\eta (1-\gamma) \left(V^{x, f^*}(\sigma) - V^{x, f^t}(\sigma) \right) + \eta\ err_t - \frac{\beta \eta^2 W^2}{2}.
\end{align*}
Rearrange it and we get
\begin{align*}
    & \quad V^{x, f^t}(\sigma) - V^{x, f^*}(\sigma)\\
    & \le \frac{1}{1-\gamma} 
    \left( \frac{1}{\eta} \E_{s \sim d_\sigma^{x, f^*}} \E_{a \sim x}  \left(  KL(f^* \| f^t) - KL(f^* \| f^{t+1}) \right)  
    -err_t + \frac{\eta \beta W^2}{2}\right)
\end{align*}
 Taking the sum, and notice that $\theta^0 = 0$
 \begin{align*}
     \frac{1}{T} \sum_{t=0}^{T-1} (V^{x, f^t}(\sigma) - V^{x, f^*}(\sigma) ) \le
     \frac{1}{1-\gamma} \left( \frac{\log{|\gA |}}{\eta T} + \frac{\eta \beta W^2}{2} - \frac{1}{T} \sum_{t=0}^{T-1} err_t \right)
 \end{align*}
\end{proof}

\begin{lemma}[Unbiased estimation]
Sample-based gradient in Algorithm~\ref{alg:online_NPG} is unbiased of $\nabla_w L(w)$ (Eq.~\ref{sec5_update_rule}).
\end{lemma}
\begin{proof} Recall the estimators in \citep[Algorithm 1, 3]{agarwal2020optimality} which provide unbiased estimations of $Q^{x, f^t}(s,a,b)$ and $d_\nu^{x, f^t}$. With little abuse of notation, we use $Q^t$, $A^t$ to represent $Q^{x, f^t}$ and $A^{x, f^t}$.
\begin{align*}
    & \quad \E_{s, a, b \sim \nu^t} \E_{v^\prime \sim f^t} [g_n]\\
    &= \E_{s, a, b \sim \nu^t} \hat{Q}(s, a, b) \nabla_{\theta} \log{f^t} (b|s) - \E_{s, a, b \sim \nu^t } \E_{v^\prime \sim f^t} \hat{Q}(s, a, b) \nabla_{\theta} \log{f^t} (v^\prime|s) \\
    &= \E_{s, a, b \sim \nu^t} Q^t (s, a, b) \nabla_{\theta} \log{f^t} (b|s) - \E_{s, a, b \sim \nu^t} V^t (s) \nabla_{\theta} \log{f^t} (b|s)\\
    &= \E_{s, a, b \sim \nu^t} A^t (s, a, b)\nabla_{\theta} \log{f^t} (b|s),
\end{align*}
hence,
\begin{align*}
    &\quad 2 \E_{s, a, b \sim \nu^t} \left[ \left(w_n^\top \nabla_{\theta} \log{f^t(b|s)}\right) \nabla_{\theta} \log{f^t(b|s)} - g_n \right]\\
    &= 2 \E_{s, a, b \sim \nu^t} \left[ w_n^\top \nabla_{\theta} \log{f^t(b|s)} - A^t(s, a, b) \right] \nabla_{\theta} \log{f^t(b|s)}\\
    &= \nabla_w L(w_n)
\end{align*}
Proof is completed.

\end{proof} 

\begin{lemma}[Bounded stat error] \label{lem_stat_error}Assume $\| \nabla_{\theta} \log{f (b|s)} \|_2 \le B$, statistical error of minimizing Eq.~\ref{sec5_update_rule} is bounded
\begin{align*}
    \E\left[ L(\hat{w}^t)\right] - L(w^*)  = \gO(\frac{1}{\sqrt{N}})
\end{align*}
\end{lemma}

\begin{proof}
    For this sample-based projected gradient descent, notice the estimated gradient is bounded by $G \coloneqq 2B (BW+\frac{1}{1-\gamma})$. \citet{shalev2014understanding} shows if setting learning rate $\alpha = \frac{W}{G\sqrt{N}}$,
\begin{align*}
    \E[L(\Bar{w})] - L(w^*) \le \frac{GW}{\sqrt{N}}
\end{align*}
\end{proof}

\begin{lemma}[Gradient norm bounded for log-linear parameterization]
Suppose $\pi_\theta (a|s) = \frac{\exp{(\theta^\top \phi_{s,a})}}{\sum_{a^\prime \in \gA} \exp{(\theta^\top \phi_{s, a^\prime})}}$ for which $\|\phi_{s,a}\|\le D$, we show $\| \nabla_{\theta} \log{\pi (a|s)} \|_2 \le B= 2D$.
\begin{proof} Proof is straight forward
\begin{align*}
    \nabla_\theta \log \pi_\theta (a|s) &= \phi_{s,a} - \frac { \phi_{s, a^\prime} e^{\theta^\top \phi_{s,a}}}{\sum_{a^\prime} e^{\theta^\top \phi_{s, a^\prime}}}\\
    &= \phi_{s,a} - \sum_{a^\prime} \phi_{s, a^\prime} P(a^\prime),
\end{align*}
where $P(a^\prime) = \frac { e^{\theta^\top \phi_{s,a}}}{\sum_{a^\prime} e^{\theta^\top \phi_{s, a^\prime}}}$. Then
\begin{align*}
    \|\nabla_\theta \log \pi_\theta (a|s)\| &\le \|\phi_{s,a}\| + \|\sum_{a^\prime} \phi_{s, a^\prime} P(a^\prime)\| \\
    &\le \|\phi_{s,a}\| + \sum_{a^\prime} P(a^\prime) \|\phi_{s, a^\prime}\|\\
    &\le 2D.
\end{align*}
Proof is completed.
\end{proof}
\end{lemma}

\begin{lemma}[Iteration error of Algorithm~\ref{alg:online_NPG}]
Set learning rate $\eta = \sqrt{\frac{2 \log{|\gA |}}{\beta T W^2}}$, $\alpha = \frac{W}{G\sqrt{N}}$, initial state-action distribution $\nu_0(s,a,b) = \nicefrac{ \sigma(s) } {|\gA|^2}$, $err_t$ in Lemma~\ref{lem4_NPG_regret} can be bounded  with sample complexity.
\begin{align*}
    |err_t|^2 &\le   \E_{s,a,b}^*
   \left[ A^{x, f^t} (s, a, b) - w^{t} \nabla_{\theta} \log{f^{t} (b|s)} \right]^2  \\
   &\le  \left\| \frac{d_\sigma^{x, f^*}\cdot x \cdot f^*}{\nu^t} \right\|_\infty    \E_{s, a, b \sim \nu^t} \left( A^{x, f^t}(s, a, b) - w^t \nabla_{\theta} \log{f^t(b|s)} \right)^2        \\
   &\le  \frac{1}{1-\gamma} \left\| \frac{d_\sigma^{x, f^*}\cdot x \cdot f^*}{\nu_0} \right\|_\infty    L(\hat{w}^t, \theta)   \\
   &\le  \frac{1}{1-\gamma} \left\| \frac{d_\sigma^{x, f^*}\cdot x \cdot f^*}{\nu_0} \right\|_\infty    L(\hat{w}^t, \theta)   \\
    &\le \frac{|\gA|^2}{1-\gamma} \left\| \frac{d_\sigma^{x, f^*}}{\sigma} \right\|_\infty L(\hat{w}^t, \theta) \\
    &\le \frac{|\gA|^2}{1-\gamma} \left\| \frac{d_\sigma^{x, f^*}}{\sigma} \right\|_\infty \left( L(\hat{w}^t) - L(w^*) + L(w^*) \right)
\end{align*}
From Lemma~\ref{lem_mismatch_coeff}, $\left\| \frac{d_\sigma^{x, f^*}}{\sigma} \right\|_\infty$ is controlled by $\gC_{\sigma, \sigma}^\prime$

Take the expectation on both sides of Lemma~\ref{lem4_NPG_regret}, summation of $err_t$ is bounded
\begin{align*}
    \E \left[ \sum_t \frac{-1}{T} err_t \right] &\le \E \left[ \frac{1}{T} \sum_t \sqrt{ \E_{s,a,b}^* \left(A^{x, f^t}(s, a, b) - w^t \nabla_{\theta} \log{f^t(b|s)}  \right)^2 } \right] \\
    & \le \E \sqrt{ \frac{1}{T} \sum_t  \E_{s,a,b}^* \left( A^{x, f^t}(s, a, b) - w^t \nabla_{\theta} \log{f^t(b|s)}  \right)^2 }, \text{ $y=\sqrt{x}$ is concave}\\
    & \le \sqrt{\frac{1}{T} \sum_t \E\left[ \E_{s,a,b}^* \left( A^{x, f^t}(s, a, b) - w^t \nabla_{\theta} \log{f^t(b|s)}  \right)^2 \right] }\\
    & \le \sqrt{\frac{|\gA|^2}{1-\gamma} \left\| \frac{d_\sigma^{x, f^*}}{\sigma} \right\|_\infty \cdot \E \left[ L(\hat{w}^t) - L(w^*) + L(w^*)\right] }\\
    & \le  \sqrt{\frac{|\gA|^2}{(1-\gamma)^2} \gC_{\sigma, \sigma}^\prime \left( \frac{GW}{\sqrt{N}} + \epsilon_{approx} \right) }
\end{align*}
 Further, $\forall 1\le j \le k-1$, it holds
\begin{align*}
    & \quad \E [ \sum_s \sigma(s) \epsilon_j(s) ]\\
    &= \E \left[ \frac{1}{T} \sum_{t=0}^{T-1}(V^{x, f^t}(\sigma) - V^{x, f^*}(\sigma)) \right]\\
    &\le \sqrt{\frac{2\log{|\gA|}\beta W^2}{T}} + \frac{|\gA|}{(1-\gamma)^2} \sqrt{\gC_{\sigma, \sigma}^\prime \left( \frac{GW}{\sqrt{N}} + \epsilon_{approx} \right)} \\
    &\le \sqrt{\frac{2\log{|\gA|}\beta W^2}{T}} + \frac{|\gA|}{(1-\gamma)^2} \sqrt{\gC_{\sigma, \sigma}^\prime \frac{GW}{\sqrt{N}}} + \frac{|\gA|}{(1-\gamma)^2} \sqrt{\gC_{\sigma, \sigma}^\prime \cdot \epsilon_{approx} }
\end{align*}
Take the expectation on both sides of Lemma~\ref{original_form}, note $\E [ \sup_{1 \le j \le k-1} \|\epsilon_j\|_{1, \sigma} ]$ is also upper bounded by the above inequality, then the proof is completed via substitution.
\end{lemma}

Combining these results, Theorem~\ref{thm:online_NPG} for online setting is concluded.\\\textbf{ [Proof for Theorem~\ref{thm:online_NPG}]} \label{online_thm_pf}
\begin{proof} Substitute $\epsilon$ and $\epsilon^\prime$,
\begin{align*}
        &\E \left[V^*(\rho) - \inf_{f} V^{x^k, f}(\rho) \right] \\
      \le &
      \frac{2(\gamma-\gamma^k) \gC_{\rho, \sigma}^{1, k, 0}}{(1-\gamma)^2} 
        \cdot \epsilon
        + \frac{(1-\gamma^k)\gC_{\rho, \sigma}^{0, k, 0}}{(1-\gamma)^2}  \cdot \epsilon^\prime
        + \frac{ 2 \gamma^k}{1-\gamma} \gC_{\rho, \sigma}^{, k+1, 0} ,
    \end{align*}
    where 
    \begin{align*}
        \epsilon &= \sqrt{\frac{2\log{|\gA|}\beta W^2}{T}} + \frac{|\gA|}{(1-\gamma)^2} \sqrt{\gC_{\sigma, \sigma}^\prime \frac{GW}{\sqrt{N}}} + \frac{|\gA|}{(1-\gamma)^2} \sqrt{\gC_{\sigma, \sigma}^\prime \cdot \epsilon_{approx} }\\
        \epsilon^\prime &= 2\sqrt{\frac{2\log|\gA| \beta {W}^2 }{T^\prime}} + 2\iota \left(\sqrt{\epsilon_{approx}^\prime} + \frac{\sqrt{G W}}{{N^\prime}^{\frac{1}{4}}}\right)
    \end{align*}
When the outer loop count $k$ is set as $K$, proof of Theorem~\ref{thm:online_NPG} is completed.
\end{proof}
    
\end{document}